%% file: main.tex
\documentclass[sigconf]{acmart}

\usepackage{hyperref}       %
\usepackage{url}            %
\usepackage{nicefrac}       %
\usepackage{microtype}      %
\usepackage{xcolor}         %

\usepackage{amsmath}
\usepackage{amsthm}
\usepackage{booktabs}
\usepackage{algorithm}
\usepackage{algpseudocode}
\usepackage{todonotes}
\usepackage{bm}

\usepackage{subfigure}
\usepackage{pgfplots}
\DeclareUnicodeCharacter{2212}{−}
\usepgfplotslibrary{groupplots,dateplot}
\usetikzlibrary{patterns,shapes.arrows}
\pgfplotsset{compat=newest}
\usepackage{float}
\usepackage[inkscapelatex=false]{svg}

\DeclareMathOperator*{\argmax}{arg\,max}
\DeclareMathOperator*{\argmin}{arg\,min}

\AtBeginDocument{%
  }

\setcopyright{acmlicensed}
\copyrightyear{2024}
\acmYear{2024}
\acmDOI{XXXXXXX.XXXXXXX}

\acmConference[AISec '24]{Proceedings of the 17th ACM Workshop on Artificial Intelligence and Security}{October 18, 2024}{Salt Lake City, UT}
\acmISBN{978-1-4503-XXXX-X/18/06}

\begin{document}

\title{Feature Selection from Differentially Private Correlations}

\author{Ryan Swope}
\authornote{Equal contribution.}
\email{Swope\_Ryan@bah.com}
\affiliation{%
  \institution{Booz Allen Hamilton}
  \city{Philadelphia}
  \state{Pennsylvania}
  \country{USA}
}

\author{Amol Khanna}
\authornotemark[1]
\email{Khanna\_Amol@bah.com}
\affiliation{%
  \institution{Booz Allen Hamilton}
  \city{Boston}
  \state{Massachusetts}
  \country{USA}}

\author{Philip Doldo}
\authornotemark[1]
\email{Doldo\_Philip@bah.com}
\affiliation{%
  \institution{Booz Allen Hamilton}
  \city{Baltimore}
  \state{Maryland}
  \country{USA}
}

\author{Saptarshi Roy}
\email{roysapta@umich.edu}
\affiliation{%
 \institution{University of Michigan}
 \city{Ann Arbor}
 \state{Michigan}
 \country{USA}}

\author{Edward Raff}
\email{Raff_Edward@bah.com}
\affiliation{%
 \institution{Booz Allen Hamilton}
 \institution{University of Maryland, \\Baltimore County}
 \city{Syracuse}
 \state{New York}
 \country{USA}}

\begin{abstract}
  Data scientists often seek to identify the most important features in high-dimensional datasets. This can be done through $L_1$-regularized regression, but this can become inefficient for very high-dimensional datasets. Additionally, high-dimensional regression can leak information about individual datapoints in a dataset. In this paper, we empirically evaluate the established baseline method for feature selection with differential privacy, the two-stage selection technique, and show that it is not stable under sparsity. This makes it perform poorly on real-world datasets, so we consider a different approach to private feature selection. We employ a correlations-based order statistic to choose important features from a dataset and privatize them to ensure that the results do not leak information about individual datapoints. We find that our method significantly outperforms the established baseline for private feature selection on many datasets. 
\end{abstract}




\received{20 February 2007}
\received[revised]{12 March 2009}
\received[accepted]{5 June 2009}

\maketitle

\section{Introduction}

Linear models remain one of the most common and widely used techniques in practice and research today. In particular, linear regression and logistic regression are straightforward to solve using either general-purpose convex solvers like Limited-memory BFGS~\cite{Liu1989} or bespoke optimizers like those provided in LIBLINEAR and other libraries~\cite{scikit-learn,10.5555/1390681.1442794}. 
In high-dimensional situations, where the number of features $d$ is greater than the number of samples $N$, the need to perform feature selection to avoid over-determined systems is particularly pertinent. While classic approaches like forward-backward selection~\cite{JMLR:v20:17-334} and mutual-information~\cite{Ross2014,wang2019fast} are still studied, the incorporation of sparsity-inducing penalties has become the predominant approach.

The induction of sparsity in the solution is useful from a pure engineering, practical deployment, and analytical understanding since the $L_1$ penalty was introduced via the ``LASSO'' regularizer by~\citet{Tibshirani1994}. The $L_1$ penalty further has provable advantages in high dimensional settings~\cite{Ng2004} that have led to significant effort in custom solvers~\cite{Yuan2012,frandi_fast_2016,Friedman2010} and the broad preference for making feature selection a joint process with the regression itself~\cite{hastie_best_2020}.

However, the widespread success of $L_1$ based optimization for joint solving of feature selection and model weights is not so clear cut when we are concerned with the privacy of the data used to build the model. In such a case, there is a need to introduce an additional framework to protect data privacy. Specifically, differential privacy is a statistical framework that guarantees data privacy in an algorithm \cite{near2021programming}. Given parameters $\epsilon$ and $\delta$, on any two datasets $D$ and $D'$ differing on one example, an approximate differentially private algorithm $\mathcal{A}$ satisfies $\Pr \left[ \mathcal{A}(D) \in O \right] \leq \exp\{ \epsilon\} \Pr \left[ \mathcal{A}(D') \in O \right] + \delta$ for any $O \subseteq \text{image}(\mathcal{A})$. Note that lower values of $\epsilon$ and $\delta$ correspond to stronger privacy. Differential privacy is typically achieved by adding calibrated amounts of noise in the mechanism or to the output of $\mathcal{A}$. 

Differentially private regression algorithms ensure that the weight of a regression does not reveal significant information about its training data. This is especially important in high-dimensional models, where the ratio of parameters to training data points is higher, and thus, the parameters can encode more information about the data points. For this reason, several differentially private high-dimensional regression algorithms have been developed. \textbf{However, these algorithms have to add noise scaled by the dimensionality of the data, which can quickly overwhelm the signal, and destroy all sparsity of the solution} \cite{khanna2024sok}. The noise added by differential privacy proves to be a significant roadblock to successful DP approaches that induce hard-zeros as the feature selection step because the noise continually knocks solutions away from exact zero components. For this reason, we choose to explore differentially private feature selection strategies which reduce the dimensionality of the dataset prior to using private regression algorithms. 

Few past works have considered differentially private feature selection, and of those that have, most produce computationally infeasible procedures \cite{roy2023computational,thakurta2013differentially}. \citet{kifer2012private} develop a ``two-stage'' method which is computationally feasible but still intensive as it requires training multiple models. It is inspired by the $L_1$ penalty as it involves training multiple $L_1$-regularized models, but a key insight we empirically demonstrate is that\textbf{ $L_1$ penalization is not algorithmically stable, leading to inconsistent performance under differential privacy}. 

A simple yet effective non-private feature selection algorithm is Sure Independence Screening (SIS), which selects the $k$ features with the highest absolute correlation to the target vector \cite{fan2008sure}. This method is also flexible, allowing users to switch out correlation with any other metric they believe to be better suited to identifying important variables. We compare the non-private SIS algorithm to a non-private version of \cite{kifer2012private} to disambiguate the impact of noise vs algorithmic instability of \cite{kifer2012private}. This test shows that SIS performs better, and thus isolates the instability of $L_1$-regularization as a major factor in the lower performance of \cite{kifer2012private}. We then provide intuition for why the SIS algorithm will perform better in the private setting than the two-stage approach and proceed to privatize SIS. \textbf{Our experiments show that private SIS performs better than two-stage on a variety of high-dimensional datasets, achieving similar or improved accuracies for $\epsilon$ in a usable range of $[1, 10]$, while being easy to reason and prove}. 

Here, we provide an overview of the following sections. In \autoref{sec:related_work}, we will review related works on non-private and private feature selection strategies. In \autoref{instability}, we will compare a non-private version of two-stage with a simple correlation-based selection approach and show that the baseline is significantly less stable under sparsity than the simple approach. In \autoref{sec:private_sis}, we will describe how to make this correlation-based selection private, and will provide mathematical arguments highlighting when our method will perform well. \autoref{sec:experiments} will detail experiments and their results, with \autoref{sec:conclusion} concluding the paper. We make special note that to implement a useful and private SIS feature selector, a private top-$k$ selection step was needed. The only method for top-$k$ selection that improved over the baseline was the canonical Lipschitz mechanism \cite{shekelyan2022differentially}, which in its original presentation is difficult to prove and understand. As an additional contribution to this work, we re-state and simplify the exposition of this valuable technique in \autoref{sec:appendix} to improve its utility. 

\section{Related Work} \label{sec:related_work}

In this section, we will describe related works on high-dimensional regression in the non-private and private settings. We aim to provide an overview of the space while focusing on correlation-based selection strategies in the non-private setting. We make note that the notions of $L_1$ based feature selection as a process independent of the optimization process has been developed in the non-private setting via ``screening rules''~\cite{rakotomamonjy_screening_2019,Larsson2021,ndiaye_gap_2017,wang_lasso_2013,Ogawa2013} that attempt to identify coefficients which will have a zero value once the optimization is done, and so can be discarded early for computational efficiency. Unfortunately, negative results have been shown for converting screening rules into a differentially private form~\cite{khanna2023challenge}. It is for this reason that we look back to the older approach of a separate feature selection process from model training to see if improved results can be obtained. Below, we will review the pertinent non-private and private high-dimensional regression literature. 

\subsection{Non-private High-dimensional Regression}

Statisticians typically constrain the structure of high-dimensional regressions. One such constraint is coefficient sparsity, which we will focus on in this paper. 

Constraining sparsity is equivalent to $L_0$-constrained or penalized regression. Indeed, it is not difficult to see that for varying $\lambda$, both of 
\begin{align}
    &\argmin_{\mathbf{w} \in \mathbb{R}^d: \ \lVert \mathbf{w} \rVert_0 \leq \lambda} \sum_{i = 1}^{N}   \ell(\mathbf{x}_i, y_i; \mathbf{w}) \\ 
    &\argmin_{\mathbf{w} \in \mathbb{R}^d} \sum_{i = 1}^{N}   \ell(\mathbf{x}_i, y_i; \mathbf{w}) + \lambda \lVert \mathbf{w} \rVert_0
\end{align}
will produce solutions of varying sparsity. However, $L_0$-constrained or penalized regression is NP-hard, making it impossible to find solutions to these problems in polynomial time \cite{wainwright2019high}. 

To make optimization feasible, $L_1$-constrained or regularized regression can be used. Here, the optimization functions 
\begin{align}
    &\argmin_{\mathbf{w} \in \mathbb{R}^d: \ \lVert \mathbf{w} \rVert_1 \leq \lambda} \sum_{i = 1}^{N}   \ell(\mathbf{x}_i, y_i; \mathbf{w}) \\ 
    &\argmin_{\mathbf{w} \in \mathbb{R}^d} \sum_{i = 1}^{N}   \ell(\mathbf{x}_i, y_i; \mathbf{w}) + \lambda \lVert \mathbf{w} \rVert_1
\end{align}
are convex and feasible. However, employing $L_1$ constraints or penalties creates a bias away from the $L_0$ solution, meaning the optimal $\mathbf{w}$ in Equations (3) and (4) will not equal those in Equations (1) or (2) \cite{wainwright2019high}. Nevertheless, Donoho and Huo and Donoho and Elad showed that the support sets of penalized $L_0$ solutions can be found through optimization with $L_1$ penalties when $\mathbf{w}$ is sufficiently sparse \cite{donoho2001uncertainty,donoho2003maximal}. 

However, even if convex optimization is feasible, it can be increasingly difficult on very high-dimensional datasets. For this reason, statisticians have developed a variety of feature selection mechanisms to remove irrelevant features prior to optimization. We describe a subset of these here. 

One of the first methods for feature selection which does not optimize a variant of Equations (3) or (4) is Sure Independence Screening (SIS) \cite{fan2008sure}. This method treats each feature as independent and measures the correlation between each feature and the target variable. The top-$k$ absolutely correlated features are retained, with the others being screened out prior to optimization. Experiments with SIS demonstrate that it works particularly well on data with independent features but can also perform reasonably well when features are mildly correlated. 

Similar to SIS are screening rules, which seek to bound the dual solutions of linear optimization problems within a compact set \cite{xiang2016screening}. This compact set can identify features which are surely not part of the feature set selected by $L_1$-regularized estimators, but can do so without performing optimization. In this way, screening rules can safely remove features prior to optimization to make the optimization problem more feasible. 

The compact sets which screening rules use can be spheres or spheres in combination with halfspaces. As the number of halfspaces employed grows, the computational complexity of screening grows, and often with marginal benefits. As a result, spherical sets are often used. Interestingly, in the case of spherical sets, there exists an equivalence between SIS and screening rules for some regularization value of $\lambda$ \cite{xiang2016screening}. This equivalence demonstrates that it is reasonable to select features for $L_1$-regularized problems based on correlations. 

\subsection{Private High-dimensional Regression}

High-dimensional regression has also been considered in differential privacy. We review some key works here, but see the cited survey by Khanna et al. for a more comprehensive view \cite{khanna2024sok}. 

Kifer et al. first considered high-dimensional differentially private regression by building a two-stage procedure \cite{kifer2012private}. The first stage of this procedure privately chooses a support set of size $k$ by splitting the data into $\sqrt{N}$ blocks of size $\sqrt{N}$ and privately computing which features are most consistently included in the supports of regression estimators built on each block. Once the algorithm finds a support set, it trains a final model relying on only the features in the support set. To the best of our knowledge, this algorithm is the only computationally efficient algorithm which employs feature selection to build high-dimensional regression estimators with differential privacy. 

Heuristically, we see two challenges with Kifer et al.'s approach. First, their algorithm is computationally intensive: it requires building $\sqrt{N} + 1$ high-dimensional regression estimators, which may not be possible if the dimensionality of the data is very high. Second, their algorithm relies on the algorithmic stability of sparse estimators: they need disjoint partitions of the dataset to agree on which features should be included in the support set. This may be an unreasonable assumption, as even in the non-private setting, sparse estimators are not algorithmically stable \cite{xu2011sparse}. Combining their requirement for algorithmically stable sparse support selection with a need for private (noisy) support selection may render the final support set ineffective. This heuristic analysis is the reason why we chose to study a procedure which did not require building regression estimators for support selection. 

Other methods for differentially private high-dimensional regression exist, and they typically rely on private optimization techniques. We do not provide details on these optimization techniques here, since this is out-of-scope for the methods and experiments in this paper. Instead, we comment that similar to the non-private case, these techniques become increasingly inefficient when the dimensionality of the data is very high, and a method to reduce the support set prior to optimization would be useful \cite{khanna2024sok}. Additionally, the utility of each of these methods relies on the dimensionality of their input data, and as the dimensionality increases, their performance will decrease \cite{raff2024scaling}. If instead these algorithms received a support set of reasonable size after feature selection, they could operate with better expected utility. 

Given the usefulness of a private feature selection mechanism, we choose to study the effectiveness of feature selection from private correlations. We empirically compare the performance of our method to the feature selection stage of the two-stage approach, as this is the only private feature selection strategy which is computationally efficient. 

\section{Evaluating the Instability of the Two-stage Approach} \label{instability}

A line of recent works on private high-dimensional regression has discussed the role of algorithmic instability in Kifer et. al.'s two-stage approach \cite{khanna2023differentially,khanna2023challenge,raff2024scaling,khanna2024sok}. However, none of these works seem to have evaluated whether this instability affects actual performance through an empirical lens. 

In this section, we run two simple experiments to show that even in the non-private setting, the two-stage approach suffers from instability compared to the simple SIS baseline. In the first experiment, we generate 100 datapoints $\mathbf{X} \sim \mathcal{N}(\mathbf{0}, \mathbf{I}_{100})$ where each $\mathbf{x}_i \in \mathbb{R}^{100}$. We then use a weight vector $\mathbf{w}_1 = \begin{bmatrix} 1 & 1 & 1 & 1 & 1 & 0 & \cdots & 0 \end{bmatrix}^\top$ with value 1 in components 1 through 5 and value 0 in components 6 through 100. Finally, we generate targets $y_i = \mathbf{x}_i^\top \mathbf{w}_1 + \epsilon_i$, where $\epsilon_i \sim \mathcal{N}(0, 0.1)$. 

With this dataset, we employed the first stage of Kifer et. al.'s two-stage method with a non-private selection step to identify a support set. We did this by splitting the data into $\sqrt{N}$ blocks of size $\sqrt{N}$ and identifying the 5 features which were most consistently included in the supports of the $\sqrt{N}$ regression estimators. We also employed the non-private SIS algorithm on this dataset. SIS chose a support set by identifying which 5 features in the dataset were most correlated with the target variable. 

\begin{figure}[t]
    \centering
    \includesvg[width=\linewidth]{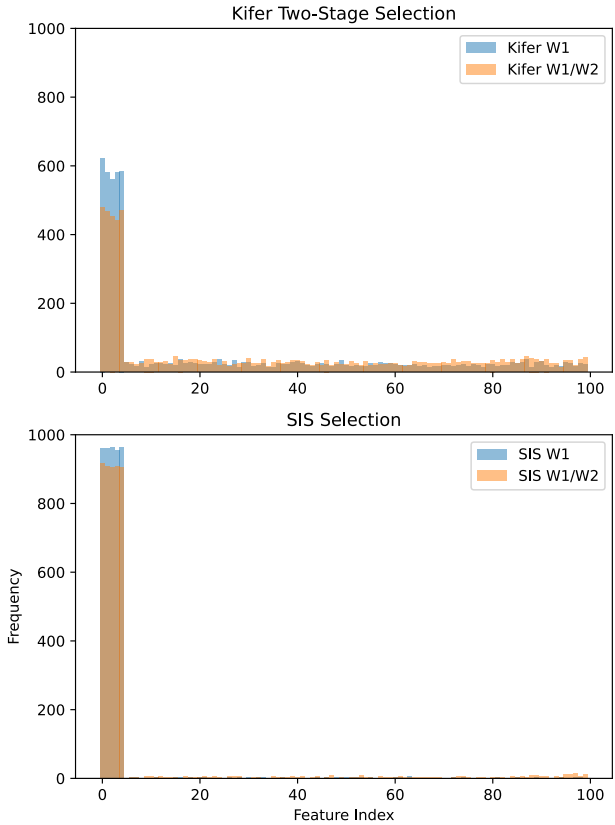}
    \caption{Occurrences of selecting feature $n$ in the two-stage and SIS algorithms. An ideal result for both experiments would select features 1 through 5 1000 times and all other features zero times. SIS is closer to the ideal result than two-stage.}
    \label{fig:kifer_stability_sec3}
\end{figure}

This experiment was repeated 1000 times, and results are included as experiment W1 in \autoref{fig:kifer_stability_sec3}. It is clear that SIS chooses the true nonzero features much more often than the two-stage method, making it a better feature-selection mechanism for this problem. 

The above experiment demonstrates that an intuitive method like SIS can outperform the two-stage method on a simple dataset. However, conditions of real-datasets are rarely so ideal. To simulate how outliers affect the two-stage and SIS methods, we employ the following experiment. We generate 100 datapoints $\mathbf{X} \sim \mathcal{N}(\mathbf{0}, \mathbf{I}_{100})$ where each $\mathbf{x}_i \in \mathbb{R}^{100}$, like above. We then use two weight vectors: $\mathbf{w}_1 = \begin{bmatrix} 1 & 1 & 1 & 1 & 1 & 0 & \cdots & 0 \end{bmatrix}^\top$ and $\mathbf{w}_2 = \begin{bmatrix} 0 & \cdots & 0 & 1 & 1 & 1 & 1 & 1 \end{bmatrix}^\top$. $\mathbf{w}_1$ the same vector described in the previous experiment, whereas $\mathbf{w}_2$ has value 0 in components 1 through 95 and value 1 in components 96 through 100. Finally, using the first 90 datapoints, we generate $\mathbf{y}_1 = \mathbf{X}_{[1:90, :]}\mathbf{w}_1 + \bm{\epsilon}_{[1:90]}$ and $\mathbf{y}_2 = \mathbf{X}_{[91:100, :]}\mathbf{w}_2 + \bm{\epsilon}_{[91:100]}$, where $\epsilon_i \sim \mathcal{N}(0, 0.1)$. When constructing our final dataset, we choose to repeatedly intersperse 9 datapoints generated from $\mathbf{w}_1$ with 1 datapoint generated from $\mathbf{w}_2$ so that each disjoint block of the two-stage algorithm receives one outlier. 

This experiment was repeated 1000 times, and the results are included in \autoref{fig:kifer_stability_sec3} as W1/W2. It is clear that even in the presence of these outliers, SIS still chooses the nonzero components corresponding $\mathbf{w}_1$ much more frequently than the two-stage method. This is desirable if we believe that datapoints generated with $\mathbf{w}_2$ should be attributed to noise, and it demonstrates that SIS can outperform the two-stage method in the presence of such noise. 

\section{Privatizing SIS} \label{sec:private_sis}

In this section, we will detail our approach to private feature selection. We will begin by describing our algorithm and go on to provide a theoretical analysis. 

\subsection{Feature Selection from Private Correlations}

Given the performance of SIS in the previous section, our algorithm employs a privatized version of SIS. As a reminder, SIS measures the correlation between each feature and the target variable. The top-$k$ absolutely correlated features are retained. To privately find the top-$k$ highest absolutely correlated features with a target variable, we need to compute private correlations. To do this, we need to bound the sensitivity of the dot product between a feature and the target vector. Unlike the non-private SIS algorithm, which centers and normalizes each feature to a variance of 1, we require that the infinity-norm of each column in the dataset is 1 and the infinity-norm of the target vector is 1.\footnote{In order to have a better approximation for private correlations, the columns in the dataset must also be centered. This is a common preprocessing technique in regression, and prior works on private regression have assumed datasets with bounded infinity-norms and centered features, see many in \cite{khanna2024sok}. We also operate in this setting, but we do recognize that performing this preprocessing in a private manner can be challenging. In our experience, a lot of privacy engineering teams will perform preprocessing in a nonprivate manner, despite this not being ideal. Very recent work has attempted to tackle preprocessing in a private manner, see \cite{hu2024provable}.} One reason we choose to privatize SIS is because it is easy to understand and can be easily adapted.

In the following steps, $\mathbf{X} \in \mathbb{R}^{N \times d}$ is the design matrix, with $\mathbf{x}_{(i)}$ representing the $i^{\text{th}}$ column (or feature) of $\mathbf{X}$. $\mathbf{y} \in \mathbb{R}^{N}$ is the vector of targets. DP-SIS works as follows in Algorithm 1:

\begin{algorithm}
\label{alg:dpsis}
\caption{DP-SIS}
\begin{algorithmic}[1]
    \Require Design matrix $\mathbf{X} \in \mathbb{R}^{N \times d}$ with infinity-norm of each column at most 1, target vector $\mathbf{y} \in \mathbb{R}^{N}$ with infinity-norm at most 1, privacy parameter $\epsilon$.
    \State Employ a private top-$k$ selection strategy with parameter $\epsilon$ given that the sensitivity of $\lvert \mathbf{x}_{(i)}^{\top}\mathbf{y} \rvert$ is 1. The ``score'' or ``quality'' of feature $i$ is measured with the absolute correlation of the feature with the target, namely $\lvert \mathbf{x}_{(i)}^{\top}\mathbf{y \rvert}$. The private top-$k$ algorithm returns $k$ indices, corresponding to the selected features. 
\end{algorithmic}
\end{algorithm}

This algorithm is simple to understand and implement, but we believe that it may outperform Kifer et al.'s two-stage approach since it employs all the data and does not rely on the algorithmic stability of sparse estimators. In the following section, we provide a theoretical understanding of this algorithm. 

Note that for DP-SIS to work it is necessary to use a high-quality private top-$k$ selection algorithm. We initially tested mechanisms in \cite{qiao2021oneshot}, \cite{gillenwater2022joint}, and \cite{durfee2019practical}, but found that they added too much noise to yield favorable results. However, we found that the canonical Lipschtiz algorithm by Shekelyan \& Loukides does work, though it is not easily accessible in its original presentation. As a significant component of our work, we re-derived and formalized the canonical Lipschtiz mechanism, presenting it in a more accessible way. This explanation is placed in the appendix, highlighting that we are not claiming Shekelyan \& Loukides innovation, but we believe that our explanation was necessary and will support other works in the future. 

\subsection{Theoretical Analysis}

To better understand when this method works well, we provide the following analysis. We seek to identify the probability that DP-SIS identifies the same features as non-private SIS. 

\begin{theorem}
    Let $\mathbf{X}$ and $\mathbf{y}$ be the design matrix and target vector after transformations in steps 1 and 2. Then $\lvert (\mathbf{X}^\top\mathbf{y})_i \rvert = \lvert \mathbf{x}_{(i)}^{\top}\mathbf{y} \rvert$. Denote $\left[ \lvert \mathbf{X}^\top\mathbf{y} \rvert \right]_j$ to be the $j^\text{th}$ largest element in $\lvert \mathbf{X}^\top\mathbf{y} \rvert$. 
    
    We are trying to find the top-$k$ highest absolutely correlated features with the target variable. Assume $\left[ \lvert \mathbf{X}^\top\mathbf{y} \rvert \right]_k - \left[ \lvert \mathbf{X}^\top\mathbf{y} \rvert \right]_{k + 1} = \xi$. Then the probability that DP-SIS exactly identifies the true top-$k$ features is at least \[1 - \exp{\left\{ k \log \left( \frac{d}{k} \right) + \log \left( c_{d, k} \right) - \frac{\xi \gamma \epsilon}{2} \right\} }, \]
    where $c_{d, k} = {d \choose k} / \frac{d^k}{k^k} \leq k$ and $\gamma$ is a hyperparameter of canonical Lipschitz which must be between 0 and 1. 
\end{theorem}

\begin{proof}
    This statement follows directly from theorems A.19 and A.20 by Shekelyan \& Loukides \cite{shekelyan2022differentially}.\footnote{Theorem A.20 of their paper has a typo: the positions of $\vec{x}_{[T]}$ and $\vec{x}_{[k]}$ should be switched. This is clear when reading the derivation of Theorems A.19 and A.20.} This theorem provides a high-probability bound for the difference between the smallest privately selected element and the true $k^\text{th}$ largest element. If this difference is less than $\xi$, we know that the smallest privately selected element is the true $k^\text{th}$ largest element, and thus the sets overlap exactly. 

    Given this, we must solve \[ \frac{2}{\gamma \epsilon} \left( k \log  \left( \frac{d}{k} \right) - \log \alpha + \left( c_{d, k} \right) \right) < \xi \] for $\alpha$. Rearranging this inequality, we find \[\alpha > \exp{\left\{ k \log \left( \frac{d}{k} \right) + \log \left( c_{d, k} \right) - \frac{\xi \gamma \epsilon}{2} \right\} }. \] Theorem A.20, we know that with probability at least $1 - \alpha$, the smallest privately selected element is less than $\xi$ far from the true $k^\text{th}$ largest element. Choosing the smallest $\alpha$ from the inequality above produces the result. 
\end{proof}

Although it is likely difficult to identify the value of $\xi$ privately, this analysis is useful to understand how the hyperparameters of DP-SIS impact its performance. As $k$ and $d$ increase, it becomes less likely for the DP-SIS to exactly identify the top-$k$ scores. This makes sense - as $k$ and $d$ increase, the algorithm must add more noise which produces a higher likelihood for random variations in selected features. Next, as $\xi$, $\gamma$, and $\epsilon$ increase, DP-SIS will be more consistent with the non-private algorithm. This also matches intuition - higher $\xi$ means greater separation between the $k$- and $k + 1$-th scores and higher $\epsilon$ means weaker privacy and less noise. Although our analysis indicates that higher $\gamma$ would also produce better results, the top-$k$ simulations done in Shekelyan \& Loukides indicates that for the average case, $\gamma = \frac{1}{2}$ performs better than $\gamma = 1$, so we use $\gamma = \frac{1}{2}$ in our experiments \cite{shekelyan2022differentially}.

\section{Experiments} \label{sec:experiments}

\input{dataset_linearity.tex}

\input{paper_plots/christensen/topk_christensen_subplots.tex}
\input{paper_plots/sorlie/topk_sorlie_subplots.tex}
\input{paper_plots/yeoh/topk_yeoh_subplots.tex}
\input{paper_plots/synth/topk_synth_subplots.tex}

\begin{figure}[t]
    \centering
    \includegraphics[width=\linewidth]{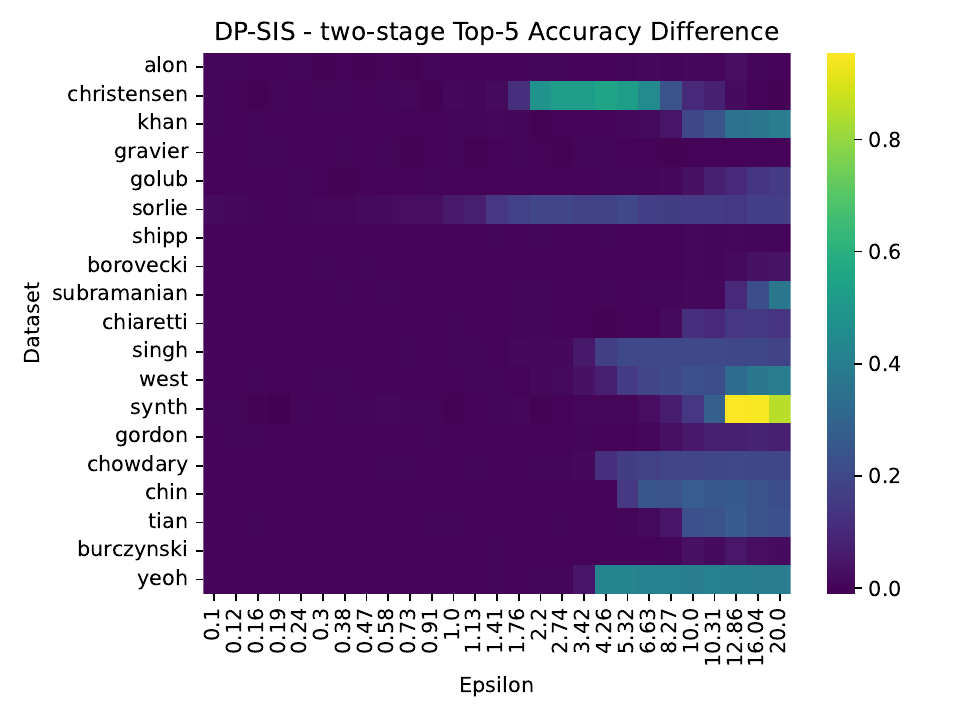}
    \caption{Difference between DP-SIS and two-stage top-$k$ for $k = 5$. On no dataset and $\epsilon$ value does two-stage significantly outperform DP-SIS, but DP-SIS significantly outperforms two-stage on ranges of $\epsilon$ values typically greater than 1.0.}
    \label{fig:heatmap_k_5}
\end{figure}

\begin{figure}[t]
    \centering
    \includegraphics[width=\linewidth]{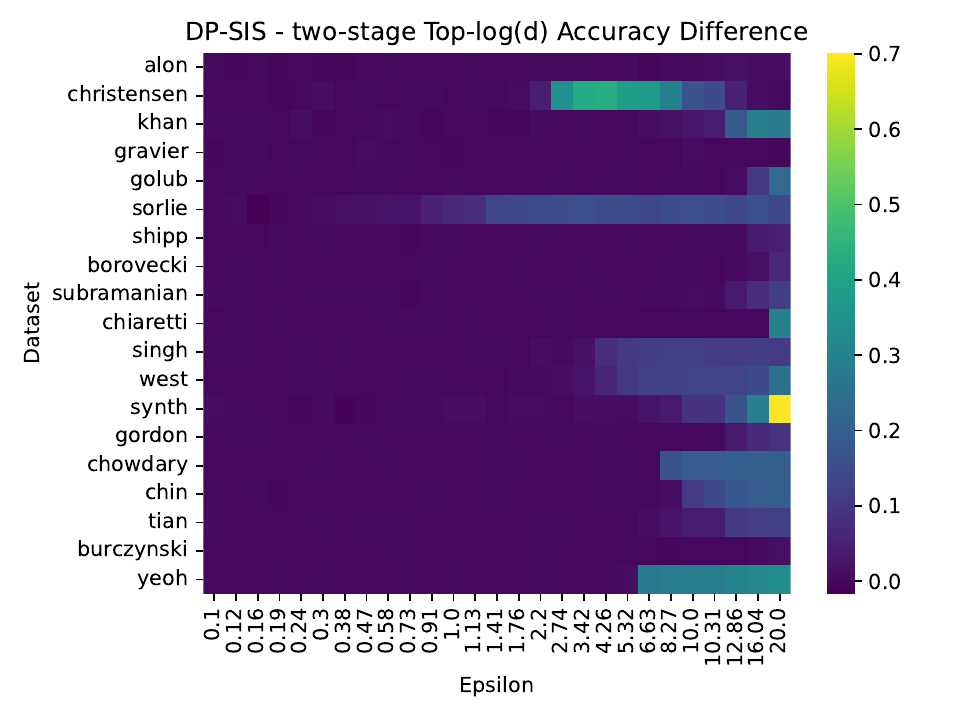}
    \caption{Difference between DP-SIS and two-stage top-$k$ accuracy for $k=\log{d}$. On no dataset and $\epsilon$ value does two-stage significantly outperform DP-SIS, but DP-SIS significantly outperforms two-stage on ranges of $\epsilon$ values typically greater than 1.0.}
    \label{fig:heatmap_k_log_d}
\end{figure} 

We used DP-SIS and the two-stage approach to perform differentially private top-\(k\) feature selection on nineteen common high-dimensional datasets and an additional synthetic dataset. 
The synthetic data was generated using the technique described in section 3.3.1 of \cite{Fan2006SureIS} with $(n,d)=(100, 2000)$. We use the usual linear model, found in \autoref{eq:linear_model}.

\begin{equation}
    \begin{aligned}
        \mathbf{y} &= \mathbf{X} \mathbf{w} + \bm{\epsilon} \\
        \mathbf{X} &\sim \mathcal{N}(\mathbf{0}, \mathbf{I}_d) \\
        \bm{\epsilon} &\sim \mathcal{N}(\mathbf{0}, 1.5\mathbf{I}_d)
    \end{aligned}
    \label{eq:linear_model}
\end{equation}

The weight vector $\mathbf{w}$ was generated by randomly selecting eight indices to serve as the non-zero values in the weight vector. The value of a non-zero index $w_i$ is given by \autoref{synth_data_gen}.

\begin{align}
    w_i &= (-1)^u (a + \left| z \right|) \label{synth_data_gen}
\end{align}
where
\begin{align*}
    u &\sim \text{Bernoulli}(0.4) \\
    z &\sim \mathcal{N}(0, 1) \\
    a &= 4 \frac{\log n}{\sqrt{n}}
\end{align*}

The synthetic dataset served as a useful benchmark as it provided a model with a known solution and therefore a known set of top-k features. For the real datasets, we used used Alon \cite{Alon1999}, Borovecki \cite{Borovecki2005}, Burczynski \cite{Burczynski2006}, Chiaretti \cite{Chiaretti2004}, Chin \cite{Chin2006}, Chowdary \cite{Chowdary2006}, Christensen \cite{Christensen:2009gu}, Golub \cite{Golub1999}, Gordon \cite{Gordon2002}, Gravier \cite{Gravier:2010bz}, Khan \cite{Khan2001}, Shipp \cite{Shipp2002}, Singh \cite{Singh2002}, Sorlie \cite{Sorlie:2001kr}, Subramanian \cite{Subramanian2005}, Tian \cite{Tian2003}, West \cite{West2001}, and Yeoh \cite{Yeoh2002} datasets. These are summarized in \autoref{tab:dataset_linearity}.\footnote{Note that we scale and center our synthetic dataset to match the conditions of DP-SIS. Many of the real datasets are already scaled and centered; for those which were not, we performed this operation.}

Three hyperparameters were used in our experiments: $\epsilon$, $k$, and $\lambda$. $\epsilon$ controlled the privacy budget: smaller values of $\epsilon$ corresponded to higher levels of privacy. $\epsilon$ values between 0.1 and 20 were tested. $k$ corresponded to the number of features to be selected. Finally, $\lambda$ was used as $L_1$ regularizer to perform non-private LASSO regression on the datasets. This is necessary in the two-stage procedure but we also used the top-$k$ components of the non-private LASSO regression as a proxy for the true top-$k$ features, since it is computationally infeasible to identify the optimal features to a $k$-sparse constrained regression problem on datasets with many features. In all presented results, $\lambda$ was set to 0.1 because when running experiments with different $\lambda$ values we found little difference in results. Note that the $\mathbf{R^2}$ scores for the non-private LASSO are listed in \autoref{tab:dataset_linearity}. 

For the synthetic dataset we selected \(k \in \{5,8\}\), as there were only 8 significant features in this synthetic dataset. For the other datasets, we tested \(k \in \{5, \lfloor \log(d) \rfloor \}\). We chose to use 5 since in many cases data analysts seek to identify a small subset of important features in a dataset, and we used $\log(d)$ since theoretical literature on feature selection often focuses on effective recovery of $\log(d)$ features. 

For each $(k, \epsilon)$ pair we performed private feature selection using the DP-SIS mechanism and the two-stage technique for one hundred trials for each dataset. To score private feature selection, we compared the number of correctly chosen top-$k$ features to $k$, producing an accuracy score. Experiments were run on an high-performance computing cluster parallelized across several nodes at a time using the SLURM HPC resource manager \cite{yoo2003slurm}. Each node had 40 CPUs and 512 GB of RAM. 

\autoref{fig:topk_christensen_subplots}, \autoref{fig:topk_sorlie_subplots}, \autoref{fig:topk_yeoh_subplots}, and \autoref{fig:topk_synth_subplots} display the performance of DP-SIS and two-stage for the Christensen, Sorlie, Yeoh, and synthetic datasets, respectively. It is expected that for small values of $\epsilon$ like $0.1$, methods will perform poorly due to the difficulty of feature selection, and DP in general, under high privacy constraints. In all plots, we can see that at very small values of $\epsilon$, both DP-SIS and two-stage have poor accuracy as expected. 

However, on each dataset, DP-SIS outperforms the two-stage mechanism for a range of intermediate $\epsilon$ values. This range of intermediate values is in line with real-world $\epsilon$ usage, indicating that our method is in-line and potentially useful for practical model building \cite{near2023guidelines}. This indicates that DP-SIS is a more useful method than the two-stage approach since it can select correct features at lower $\epsilon$ values than the two-stage mechanism, meaning that it can produce useful results while maintaining the privacy of individuals in the datasets. 

\autoref{fig:heatmap_k_5} and \autoref{fig:heatmap_k_log_d} show that this trend holds for more datasets. These heatmaps show that the two-stage method does not significantly outperform DP-SIS on any dataset, and for many datasets, DP-SIS has much better accuracy than the two-stage method for some $\epsilon$ values. However, some datasets like Alon and Borovecki do not show a difference between DP-SIS and two-stage for any $\epsilon$ value. This is likely because these datasets are higher dimensional and have less stable selected features than other datasets. Indeed, we found that DP-SIS outperforms the two-stage method to a greater degree when there is a stronger linear relationship between the features and the target. For example, DP-SIS outperforms two-stage around $\epsilon=5$ on the Christensen dataset but not on Alon despite the fact that they have comparable sizes ($(217, 1413)$ for Christensen and $(62, 2000)$ for Alon). However, linear regression achieves an $R^2$ of $0.7870$ on Christensen, but only $0.4822$ for Alon.

We also note that even when the accuracy of selected features is poor, DP-SIS achieves a high level of accuracy on the order statistics themselves since it is built on the canonical Lipschitz mechanism. We demonstrate this when examining the $TOP$, $GREAT$, and $GOOD$ performance of DP-SIS. These metrics are defined in \cite{shekelyan2022differentially} as: 

\begin{enumerate}
    \item $TOP$: DP-SIS selects all $k$ of the true top-k order statistics
    \item $GREAT$: DP-SIS selects all of the top $\frac{k}{10}$ order statistics and the rest come from the top $\frac{11k}{10}$
    \item $GREAT$: DP-SIS selects all of the top $\frac{k}{100}$ order statistics and the rest come from the top $\frac{3k}{2}$
\end{enumerate}
For example, if $k=200$, DP-SIS performs $GREAT$ if it selects all of the top 20 true top-k order statistics, and the remaining 180 features are selected from the top 220.

\begin{figure}[!h]
    \centering
    \includegraphics[width=\linewidth]{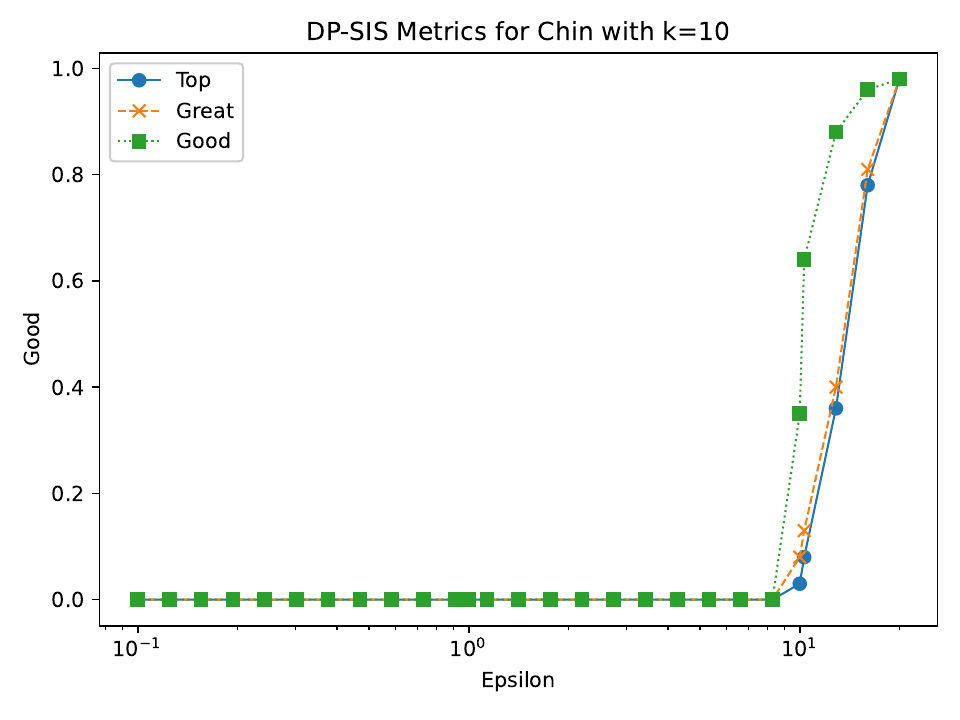}
    \caption{Performance of DP-SIS on $TOP$, $GREAT$, and $GOOD$ metrics. This graph demonstrates that the DP-SIS selection procedure is accurate with respect to order statistics, meaning that if a better representation of feature important than correlation can be engineered, the private selection procedure can be applied with expectations of good performance.}
    \label{fig:dpsis_tgg}
\end{figure}

DP-SIS only uses the order statistics during its top-$k$ selection process, and it can perform well in selecting $GREAT$ and $GOOD$ features on high-dimensional datasets. \autoref{fig:dpsis_tgg} gives an example of this on the Chin dataset. While its top-$10$ accuracy is around 20\% at $\epsilon=20$, the figure demonstrates that it performs very well on the $TOP$, $GREAT$, and $GOOD$ metrics. This indicates that the lower accuracy is a function of the correlation metric used in DP-SIS, and if another metric which better predicts feature importance is used, DP-SIS can be used effectively for private top-$k$ selection. 

\section{Conclusion} \label{sec:conclusion}

This paper seeks to improve upon the state-of-the-art in computationally feasible differentially private feature selection for high-dimensional linear regression. We identify that the current computationally feasible method, the two-stage approach, requires building $\sqrt{N}$ estimators on disjoint partitions of a dataset to identify the most commonly selected features in these estimators. While this is computationally feasible, it is still difficult since it requires building many high-dimensional estimators. Additionally, it requires high-dimensional sparse regression estimators to be algorithmically stable, which they are not. 

SIS is an alternative method for high-dimensional feature selection. It selects the $k$ features which are most correlated with the target vector. By making SIS private with a differentially private top-$k$ selector, we can develop a differentially private feature selector based on SIS. We choose to use DP-SIS based on its superior performance. 

SIS outperforms the two-stage method on most datasets in reasonable ranges of $\epsilon$. Additionally, DP-SIS can be used with any feature selection metric, making it flexible to improved metrics for feature selection developed in the future. 

Finally, we end with a final comment. This paper explores employing a simple metric for feature selection in differential privacy, and pits it against a more complicated mechanism. The simpler metric works better despite it not being able to identify cases in which subsets of features are individually weakly correlated with the target but jointly strongly correlated with the target. This is because when using differential privacy, there is a constant tug-of-war between employing more expressive methods with more noise and simpler methods with less noise. In the case presented in this paper, feature selection had higher accuracy when a simpler method was selected which was more stable and required less noise.

\bibliographystyle{ACM-Reference-Format}
\bibliography{main,raffRefs}

\appendix

\section{Appendix} \label{sec:appendix}
\subsection{Canonical Lipschitz Mechanism for Top-\(k\) Selection} 
The original paper introducing the canonical Lipschitz mechanism \cite{shekelyan2022differentially} has received little attention despite its strong results. We believe that this is in part due to the paper's presentation. In this appendix, we will attempt to review and clarify the main ideas behind using the canonical Lipschitz mechanism for \(\epsilon\)-differentially private top-\(k\) selection.  

Suppose we have a dataset containing information about some individuals. For example, in the regression context we can imagine that our design matrix \(X \in \mathbb{R}^{N \times d}\) is a dataset containing information about \(N\) individuals, with each row corresponding to a single individual. Given a dataset \(\hat{x} \in \mathbb{X}\) we can define a scoring function \(f : \mathbb{X} \to \mathbb{R}^d\) which returns a vector of \(d\) scores given some dataset in \(\mathbb{X}\).
The top-\(k\) problem is concerned with selecting a set of \(k \in \mathbb{Z}^{+}\) scores from a set of \(d > k\) scores in a differentially private manner such that the selected set is ``close'' to the true set of \(k\) largest scores, in some sense, with reasonably high probability.

The canonical Lipschitz mechanism aims to solve this problem by considering all \({d}\choose{k}\) possible \(k\)-subsets of the score indices \(\{1, ..., d\}\), 
assigning a value to each subset, and returning the subset with the largest value. The value assigned to a \(k\)-subset is the sum of two terms: a deterministic utility term (which is a negated loss term) and a randomized noise term. The utility and noise are scaled in specific ways such that \(\epsilon\)-differential privacy is achieved. The intuition is that the utility term will be larger (i.e. less negative) for index sets that are ``closer'' to the true top-\(k\) index set and each noise term is large enough such that differential privacy is maintained. 

\subsection{Top-\(k\)}

Suppose we have a dataset \(\hat{x} \in \mathbb{X}\) and a score function \(f : \mathbb{X} \to \mathbb{R}^{d}\) where each component of the score function has sensitivity \(\Delta_f\) meaning that \(|f_i(\hat{x}_1) - f_i(\hat{x_2})| \leq \Delta_f\) for \(i \in \{1, ..., d\}\) for all \(\hat{x}_1, \hat{x}_2 \in \mathbb{X}\) that differ by exactly one individual. Let \(x \in \mathbb{R}^{d}\) be a vector of normalized scores so that \(x_i = f_i(\hat{x})/\Delta_f\) for \(i \in \{1, ..., d\}\). The score vector \(x\) has descending order statistics 
\(x_{[1]} \geq ... \geq x_{[d]}\) and, just as in \cite{shekelyan2022differentially}, we let \(j_1, ..., j_d\ \in \{1,...,d\}\) be the indices such that \(x_{j_1} = x_{[1]}, ..., x_{j_d} = x_{[d]}.\) The index set \(\{j_1, ..., j_k\}\) is the true top-\(k\) index set.

Let \(\mathbb{Y}\) be the set of all size-\(k\) subsets of \(\{1, ..., d\}\). This is the set of size-\(k\) index sets and ultimately we will use the canonical Lipschitz mechanism to randomly select an index set that approximates the top-\(k\) index set by assigning a value to every \(y \in \mathbb{Y}\). We note that the ``canonical'' Lipschitz mechanism is a special case of the more general Lipschitz mechanism that uses a specific  ``canonical`` loss function. Below we describe the Lipschitz mechanism for top-\(k\) selection, but additionally note that the Lipschitz mechanism for top-\(k\) is a special case of the general Lipschitz mechanism described in \cite{shekelyan2022differentially} which contains an additional parameter \(\kappa\) (not to be confused with \(k\)) and returns the top-\(\kappa\) highest valued objects. In the case of top-\(k\) selection, we use \(\kappa=1\) as each object we are concerned with is a size-\(k\) index set. With \(\kappa > 1\), the Lipschitz mechanism would (a bit confusingly) return the top-\(\kappa\) size-\(k\) subsets that have the highest value (in other words, the top-\(\kappa\) differentially private approximations of the top-\(k\) index set).

\begin{definition}[Lipschitz Mechanism, \(\kappa=1\)]
Let \(F\) be a cumulative distribituion function such that \(\log(1 - F(x))\) is \(1\)-Lipschitz continuous and let \(F^{-1}\) be the corresponding inverse cumulative distribution function. Let \(\mathbb{Y}\) be the domain the mechanism selects its output from and let \(U_y \sim Unif(0, 1)\) for each \(y \in \mathbb{Y}\) be i.i.d.
Let \(\hat{x} \in \mathbb{X}\) be a dataset and \(f : \mathbb{X} \to \mathbb{R}^{d}\) be a score function with component-wise sensitivity \(\Delta_f\) so that we have the (normalized) score vector \(x \in \mathbb{R}^d\) such that \(x_i = f_i(\hat{x})/\Delta_f\) for \(i \in \{1, ..., d\}\). Let \(\text{LOSS}(\cdot|\cdot) : \mathbb{Y} \times \mathbb{R}^{d} \to \mathbb{R}_{\geq 0}\) be a loss function with sensitivity \(\Delta_{\text{LOSS}}\).
Let \(\epsilon > 0\) be the privacy loss parameter. The output of the Lipschitz mechanism is \[Y = \argmax_{y \in \mathbb{Y}} \left( -\frac{\epsilon}{2 \Delta_{\text{LOSS}}} \text{LOSS}(y | x) + F^{-1}(U_y) \right).\]
\end{definition}

In the top-\(k\) setting, the selection domain \(\mathbb{Y}\) will be the set of size-\(k\) subsets of the index set \(\{1, ..., d\}\). However, notice that \(\mathbb{Y}\) is exponentially large, containing \({d}\choose{k}\) elements, each of which we would need to compute the utility and noise terms of if we were to naively implement the Lipschitz mechanism, which would be prohibitively expensive. The authors of the original paper \cite{shekelyan2022differentially} cleverly get around this by making a specific choice of loss function such that many elements of \(\mathbb{Y}\) share the same loss value and \(\mathbb{Y}\) can ultimately be partitioned into only \(O(dk)\) utility classes. Using the fact that \(U_0^{1/m}\) is equal in distribution to \(\max \{U_1, ..., U_m\}\) for i.i.d. \emph{standard} uniform variables \(\{U_i\}_{i=0}^{m}\) along with the fact that \(F^{-1}\) is an increasing function, we only need to generate \(F^{-1}(U^{1/m})\) to find the maximal noise term (which is all we care about since all elements in the utility class share the same loss value) rather than generating \(m\) different samples and then taking the maximum. 

\subsection{Canonical Loss Function}

Let \(y \in \mathbb{Y}\) be some size-\(k\) index set and \(x\in\mathbb{R}^{d}\) be some normalized score vector corresponding to some score function \(f\) and dataset \(\hat{x} \in \mathbb{X}\) such that \(x_i = f_i(\hat{x})/\Delta_f\) for \(i \in \{1, ..., d\}\). The so-called canonical loss function that \cite{shekelyan2022differentially} introduces is given by 
\[\text{LOSS}(y | x) = \min_{v \in \text{OPT}^{-1}(y)} \|x-v\|_{\infty}\]
where \(OPT^{-1}(y)\) is the set of all \(v \in \mathbb{R}^d\) whose \(k\) largest values have the indices \(y = \{y_1, ..., y_k\}\). We note that this loss function implicitly depends on the dataset \(\hat{x}\) through its explicit dependence on the score vector \(x\) and it can be shown (see \cite{shekelyan2022differentially}) that this loss function has a sensitivity of \(\Delta_{\text{LOSS}}=1\), a result which relies on the fact that the score vector was appropriately normalized by \(\Delta_f\). 

This canonical loss function is special because it allows us to partition \(\mathbb{Y}\) into disjoint utility classes where all elements in a given utility class have the same canonical loss value. The reason this partition can be made is because it can be shown that 
\[\text{LOSS}(y | x) = \min_{v \in \text{OPT}^{-1}(y)} \|x-v\|_{\infty} = \frac{x_{[h+1]} - x_{[t]}}{2}\]
where \(h\) is the largest integer strictly less than \(k\) such that all of \(\{j_1, ..., j_h\}\) are in \(y\) and \(t\) is the smallest integer greater than or equal to \(k\) such that all of \(\{j_{t+1}, ..., j_{d}\}\) are \emph{not} included in \(y\).
Each utility class \(\mathcal{C}_{h, t}\) is parameterized by two integers: \(h \in \{0, 1, ..., k-1\}\) and \(t \in \{k, ..., d\}\) where \(t=k\) is only allowed if \(h=k-1\) in order to ensure that each element of \(\mathcal{C}_{h, t}\), defined below, contains a total of \(k\) indices.
\begin{definition}[Utility Class]
The utility class \(\mathcal{C}_{h, t}\) is the set of all subsets of the form \(\{j_1, ..., j_h\} \cup \mathcal{B} \cup \{j_t\}\) with \(\mathcal{B} \subseteq \{j_{h+1}, ..., j_{t-1} \}\) and \(h + |\mathcal{B}| + 1 = k\). If \(h=0\), then we define \(\{j_1, ..., j_h\}\) to be the empty set.
\end{definition}
Borrowing terminology from \cite{shekelyan2022differentially}, each \(y \in \mathcal{C}_{h, t}\) comprises of a ``head'' \(\{j_1, ..., j_h\}\), a ``body'' \(\mathcal{B}\), and a ``tail'' \(\{j_t\}\). Every \(y \in \mathcal{C}_{h, t}\) contains the true top-\(h\) indices and the remaining \(k-h\) indices are all within the true top-\(t\) indices, so \(j_{h+1}\) is the best index not included in \(y\) and \(j_{t}\) is the worst index included in \(y\). Consequently, we see that the canonical loss \(\text{LOSS}(y|x)\) is constant across all \(y \in \mathcal{C}_{h, t}\) because the loss only depends on the best missing score \(x_{[h+1]}\) and the worst included score \(x_{[t]}\). We note that the canonical loss function can be generalized to the form \(\text{LOSS}(y|x) = (1-\gamma)x_{[h+1]} - \gamma x_{[t]}\) for \(\gamma \in [0,1]\), where above we had the case where \(\gamma=1/2\). This generalized loss function also has a sensitivity of 1 (provided a normalized score vector) and \cite{shekelyan2022differentially} consider additional optimizations that can be made to the algorithm in the \(\gamma=1\) case, but we only discuss the \(\gamma \in [0,1)\) case for simplicity.

\subsection{Implementation}

The Lipschitz mechanism in the top-\(k\) setting returns the size-\(k\) index set which maximizes, over all \(y \in \mathbb{Y}\), a value which is the sum of a deterministic utility term and a random noise term. To efficiently implement the canonical Lipschitz mechanism, we only need to iterate through every utility class \(\mathcal{C}_{h,t}\) and compute the maximal noise term for each utility class because every \(y \in \mathcal{C}_{h,t}\) has the same utility and the additive noises are i.i.d. so every \(y \in \mathcal{C}_{h,t}\) is equally likely to receive the largest noise term. Once we have the \(y \in \mathbb{Y}\) with the largest value for each class, we simply return the one with the largest value across all classes. That is, for each \(\mathcal{C}_{h,t}\) we compute 
\[Y_{h,t} = \argmax_{y \in \mathcal{C}_{h,t}}  \left( -\frac{\epsilon}{2 \Delta_{\text{LOSS}}} \text{LOSS}(y | x) + F^{-1}(U_y) \right)\] and the canonical Lipschitz mechanism ultimately returns the \(Y_{h,t}\) with the largest value across all possible \(h\) and \(t\). There are only \( k(d-k) + 1 = dk - k^2 + 1 = O(dk)\) different utility classes, but some utility classes can contain many elements. It is easy to see that the size of a utility class is determined by the number of possible bodies it can have and each body contains \(|\mathcal{B}| = k - h - 1\) elements and there are \((t-1) - (h+1) + 1 = t - h - 1\) possible values that can be chosen from to form a given body, thus \(|\mathcal{C}_{h, t}| = {{t-h-1}\choose{k-h-1}}\). This means that naively we would have to generate exponentially many standard uniform noise terms for a single utility class when \(|\mathcal{C}_{h, t}| = {{t-h-1}\choose{k-h-1}}\) is large, but as we noted earlier we can get around this by only sampling \(U^{1/|\mathcal{C}_{h,t}|}\) with \(U \sim Unif(0,1)\) as this is equal in distribution to the maximum over all of the standard uniform noises and \(F^{-1}\), which is applied to the noise, is an increasing function. This allows us to perform the canonical Lipschitz mechanism for top-\(k\) selection in only \(O(dk)\) time, which is a dramatic reduction from the naive exponential complexity, noting that we can compute the binomial coefficients \(|\mathcal{C}_{h,t}|\) efficiently by updating them as we iterate through \(h\) and \(t\) using the fact that \({n+1 \choose j+1} = {n \choose j} \cdot \frac{n+1}{j+1}\) for some \(0 \leq j \leq n\), see Algorithm \ref{alg}. 

\begin{algorithm}
    \caption{Canonical Lipschitz Mechanism for Top-\(k\)}
    \label{alg}
\begin{algorithmic}[1]
    \State \textbf{Input:} dataset \(\hat{x} \in \mathbb{X}\), score function \(f : \mathbb{X} \to \mathbb{R}^d\) with component-wise sensitivity \(\Delta_f\), subset size \(k \in \{1, ..., d-1\}\), privacy loss \(\epsilon \geq 0\), inverse CDF \(F^{-1}\), loss parameter \(\gamma \in [0, 1)\)
    
    \State Let \(x_i = f_i(\hat{x})/\Delta_f\) for all \(i \in \{1, ..., d\}\) \Comment{Define normalized score vector \(x\)}
    
    \State Compute descending order statistics \(x_{[1]} \geq ... \geq x_{[d]}\) \Comment{Sort \(x\) in \(O(d \log d)\) time}
    
    \State Let \(U_{h,t} \sim Unif(0,1)\) for all \(h \in \{0,...,k-1\}, t \in \{k,...,d\}\) be i.i.d.
    
    \State Let \(\epsilon_1 = (1-\gamma)\epsilon\) and \(\epsilon_2 = \gamma \epsilon\) \Comment{Note that \(\epsilon_1 + \epsilon_2 = \epsilon\)}
    
    \State \textbf{Initialize} \(H = k-1, T = k\) \Comment{Initialize the utility class \(\mathcal{C}_{H,T}\) to only include the true top-\(k\)}
    
    \State \(L_{H,T} = \frac{\epsilon_2 - \epsilon_1}{2} x_{[k]}\) \Comment{Initial utility: \(- \frac{\epsilon}{2 \Delta_{\text{LOSS}}}\text{LOSS}(y|x)\) for \(y \in \mathcal{C}_{k-1,k}. \)}

    \State \(X_{H,T} = F^{-1}(U_{H,T})\) \Comment{Random noise corresponding to \(\mathcal{C}_{k-1, k}\)}

    \State \textbf{Initialize} \(v = L_{H,T} + X_{H,T}\) \Comment{Value for \(\mathcal{C}_{k-1,k}\), value = utility + random noise}

    \For{\(t \in \{k+1, ..., d\}\)}
    \State \textbf{Initialize} \(h = k-1\)
    \State \textbf{Initialize} \(m = 1\) \Comment{Initialize size of utility class \(m=|\mathcal{C}_{k-1,t}| = {t-k \choose 0}\) = 1}
        \While{\(h \geq 0\)}
            \If{\(h < k-1\)}
                \State Update \(m = m \cdot \frac{t - h - 1}{k - h - 1}\) \Comment{Efficiently update binomial coefficient} %
            \EndIf

            \State \(L_{h,t} = \frac{\epsilon_2}{2} x_{[t]} - \frac{\epsilon_1}{2} x_{[h+1]} \) \Comment{Utility term for \(\mathcal{C}_{h,t}\)}
            
            \State \(X_{h,t} = F^{-1}(U_{h,t}^{1/m})\) \Comment{Random noise term for \(\mathcal{C}_{h,t}\)}

            \State \(v' = L_{h,t} + X_{h,t}\) \Comment{Value for \(\mathcal{C}_{h,t}\)}

            \If{\(v' > v\)}
                \State Update \(v=v', H=h, T=t\) \Comment{Update \(\mathcal{C}_{H,T}\) to be \(\mathcal{C}_{h,t}\) if \(\mathcal{C}_{h,t}\) has a larger value}
            \EndIf

            \State Update \(h = h-1\)
            
        \EndWhile
    \EndFor
    \State Return a random index set from \(\mathcal{C}_{H,T}\)
\end{algorithmic}
\end{algorithm}

\subsection{Standard Exponential Noise Generation}

The original paper \cite{shekelyan2022differentially} achieved the best results when using the inverse cumulative distribution function corresponding to a standard exponential distribution so that \(F^{-1} = - \log(1-x)\). We discuss some implementation details when using this choice of \(F^{-1}\) and also discuss some useful properties of the noise distribution that we observed.

As the binomial coefficients \(|\mathcal{C}_{h,t}|\) can be very large, care must be taken when computing the noise term \(F^{-1}(U^{1/|\mathcal{C}_{h,t}|}) = - \log(1 - U^{1/|\mathcal{C}_{h,t}|})\) numerically. For large \(|\mathcal{C}_{h,t}|\), the quantity \(U^{1/|\mathcal{C}_{h,t}|}\) will usually be very close to 1 and so the argument of the logarithm will be very close to zero and numerical precision becomes important so that we avoid adding an artificially large noise term and consequently return bad approximations of the true top-\(k\)\ index set. To address the issues of numerical stability, we Taylor expanded the function \(g(y) = 1 - u^y\) for \(u \in [0,1]\) and \(y \geq 0\) about \(y=0\) and kept at least a quadratic approximation so that \(g(y) \approx -y \log(x) - \frac{1}{2} y^2 \log^{2}(x)\). Given that the approximation is only good when \(y\) is sufficiently close to zero, we only used it when \(y \leq 10^{-8}\) and simply used \(1 - u^y\) otherwise, which worked well in practice. 

A useful property of the noise term \(X_{h,t} = F^{-1}(U^{1/|\mathcal{C}_{h,t}|}) = - \log(1 - U^{1/|\mathcal{C}_{h,t}|})\) with \(U \sim Unif(0,1)\) is that its expected value is a harmonic number. In particular, 
\[\mathbb{E}[X_{h,t}] = \sum_{k=1}^{|\mathcal{C}_{h,t}|} \frac{1}{k}.\]

\begin{theorem}
The expected value of \(X_m = - \log(1 - U^{1/m})\) with \(U \sim Unif(0,1)\) and \(m \in \mathbb{Z}^{+}\) is the \(m^{\rm th}\) harmonic number \(H_m = 1 + 1/2 + 1/3 + \cdots + 1/m\).
\end{theorem}
\begin{proof}
It is clear that 
\[H_{m} = \int_{0}^{1} \frac{1-x^m}{1-x} dx\]
because \(\frac{1-x^m}{1-x} = 1 + x + x^2 + \cdots x^{m-1}\) for \(x \neq 1\). Integrating by parts, we see
\[H_m = \int_{0}^{1} \frac{1-x^m}{1-x} dx = -(1-x^m) \log(1-x) \bigg\rvert_{0}^{1} - \int_{0}^{1} \log(1-x) mx^{m-1} dx\]
where the first term can be shown to be zero with an application of L'Hôpital's rule. It is also true that 
\[\mathbb{E}[X_{m}] = - \int_0^1  \log(1-u^{1/m}) du = - \int_{0}^{1} \log(1-x) mx^{m-1} dx \] using a substitution of \(x=u^{1/m}\). It follows that \(\mathbb{E}[X_m] = H_m\).
\end{proof}

This fact is useful because it allows us to quantify the average magnitude of the maximal noise added to each utility class which gives us a rough idea of how large our utility terms need to be in order to compete with the noise. This information can help inform whether or not a chosen score function will yield good results for some dataset. For example, the largest utility class \(\mathcal{C}_{0,d}\) will have the largest expected noise term and one can easily compute \(H_m\) with \(m = {d-1 \choose k-1}\) to get a rough sense of how large the noise is. If the utility term is too small, then the noise will dominate and the quality of the returned index set will be worse on average. We note that very good approximations of harmonic numbers can be efficiently computed using well-known asymptotic expansions.

\end{document}

%% file: dataset_linearity.tex
\begin{table}[t]
    \centering
    \begin{tabular}{lccc} \toprule
        \textbf{Dataset} & \textbf{n} & \textbf{p} & $\mathbf{R^2}$ \\
        \midrule
        Alon & 62 & 2000 & 0.4822 \\
        Borovecki & 31 & 22283 & 0.9011 \\
        Burczynski & 127 & 22283 & 0.2175 \\
        Chiaretti & 128 & 12625 & 0.2705 \\
        Chin & 118 & 22215 & 0.5879 \\
        Chowdary & 104 & 22283 & 0.8338 \\
        Christensen & 217 & 1413 & 0.7870 \\
        Golub & 72 & 7129 & 0.6819 \\
        Gordon & 181 & 12533 & 0.2432 \\
        Gravier & 168 & 2905 & 0.0689 \\
        Khan & 63 & 2308 & 0.6961 \\
        Shipp & 77 & 7129 & 0.4204 \\
        Singh & 102 & 12600 & 0.6449 \\
        Sorlie & 85 & 457 & 0.9300 \\
        Su & 102 & 5564 & 0.9602 \\
        Subramanian & 50 & 10100 & 0.5340 \\
        Tian & 173 & 12625 & 0.0080 \\
        West & 49 & 7129 & 0.7044 \\
        Yeoh & 248 & 12625 & 0.4611 \\ \bottomrule
    \end{tabular}
    \caption{Summary of datasets. $\mathbf{R^2}$ are included for nonprivate LASSO regressions with $\mathbf{\lambda=0.1}$. The $\mathbf{R^2}$ indicate how ``sparsely linear'' a dataset is.}
    \label{tab:dataset_linearity}
\end{table}

%% file: paper_plots/christensen/topk_christensen_subplots.tex
\begin{figure}[t]
    \centering
    \begin{subfigure}{}
        \includegraphics[width=0.45\textwidth]{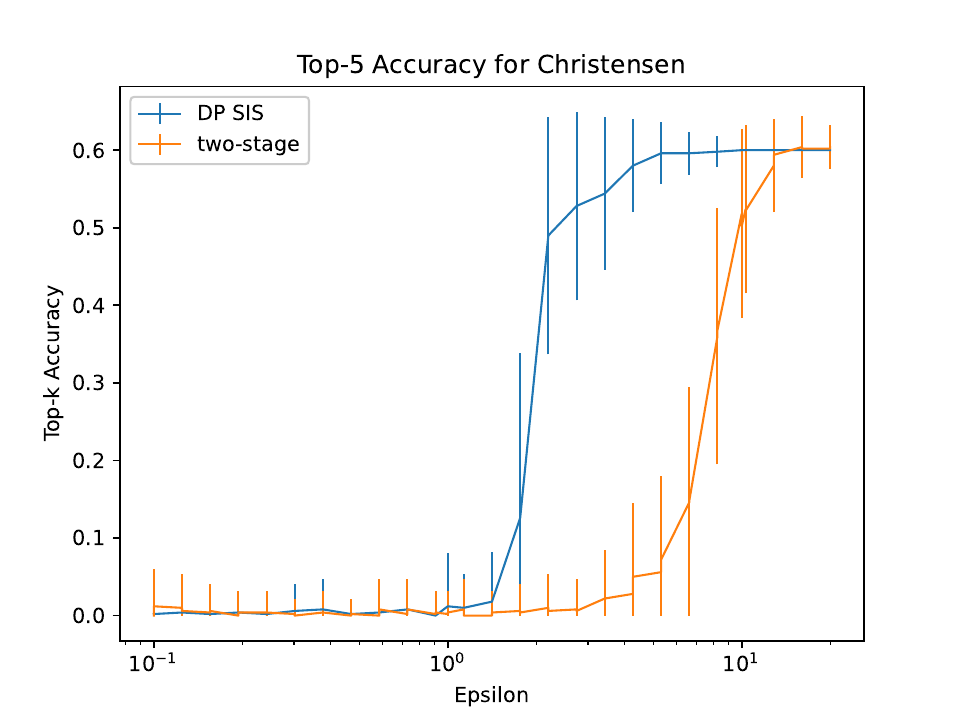}
    \end{subfigure}
    \hfill
    \begin{subfigure}{}
        \includegraphics[width=0.45\textwidth]{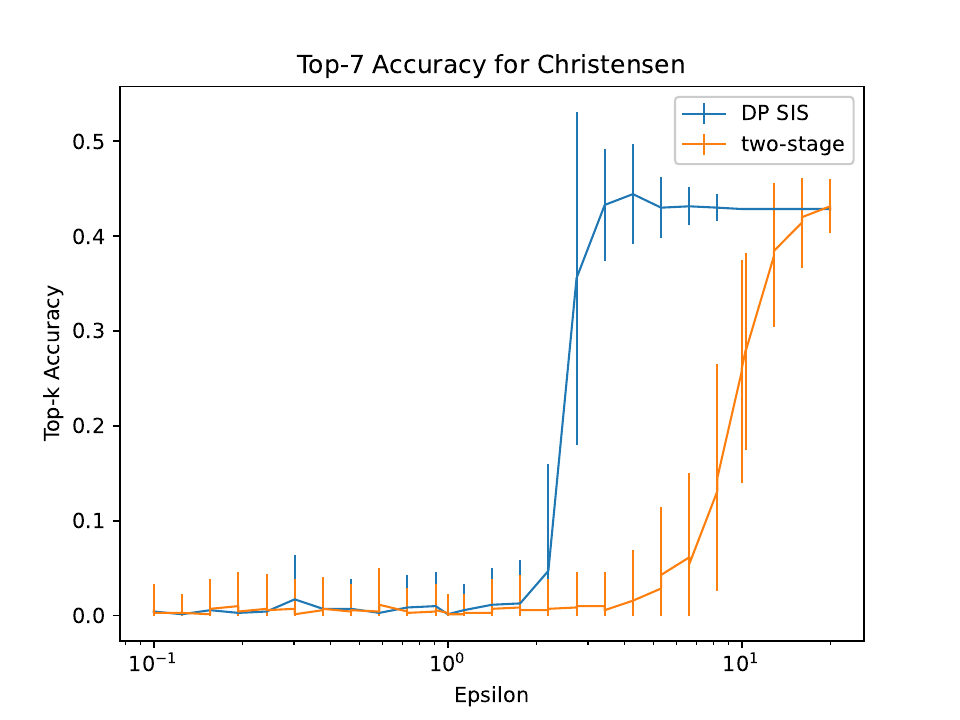}
    \end{subfigure}

    \caption{Top-$k$ accuracy of models fit on features selected from on the Christensen dataset. DP-SIS outperforms the two-stage mechanism on $\epsilon$ values between $10^0$ and $10^1$, which are commonly used for private computation \cite{near2023guidelines}.}
    \label{fig:topk_christensen_subplots}
\end{figure}

%% file: paper_plots/sorlie/topk_sorlie_subplots.tex
\begin{figure}[t]
    \centering
    \begin{subfigure}{}
        \includegraphics[width=0.45\textwidth]{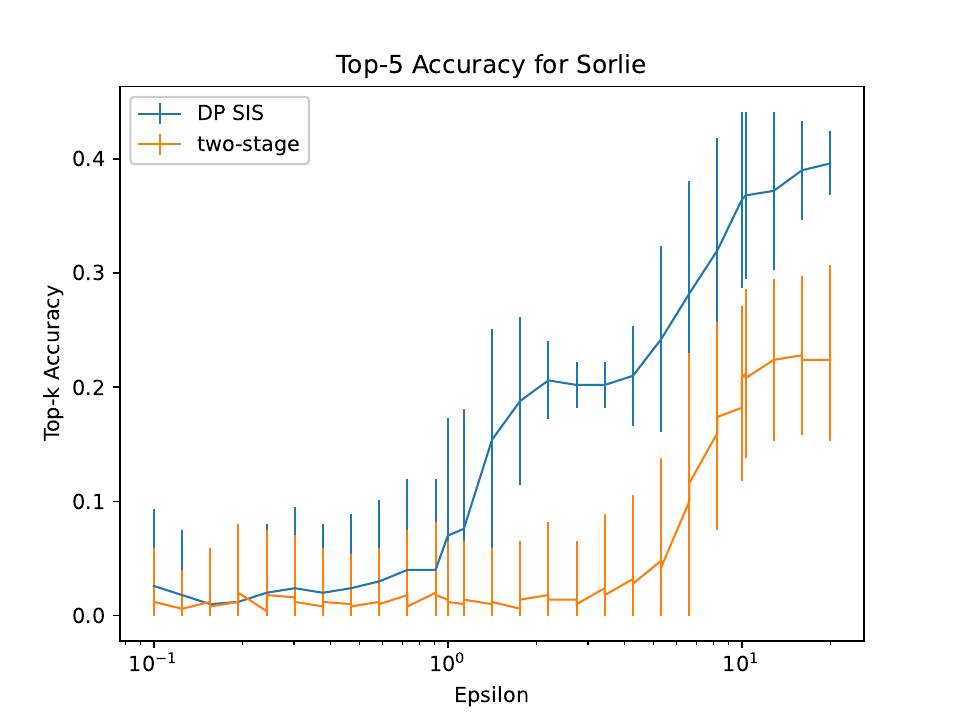}
    \end{subfigure}
    \hfill
    \begin{subfigure}{}
        \includegraphics[width=0.45\textwidth]{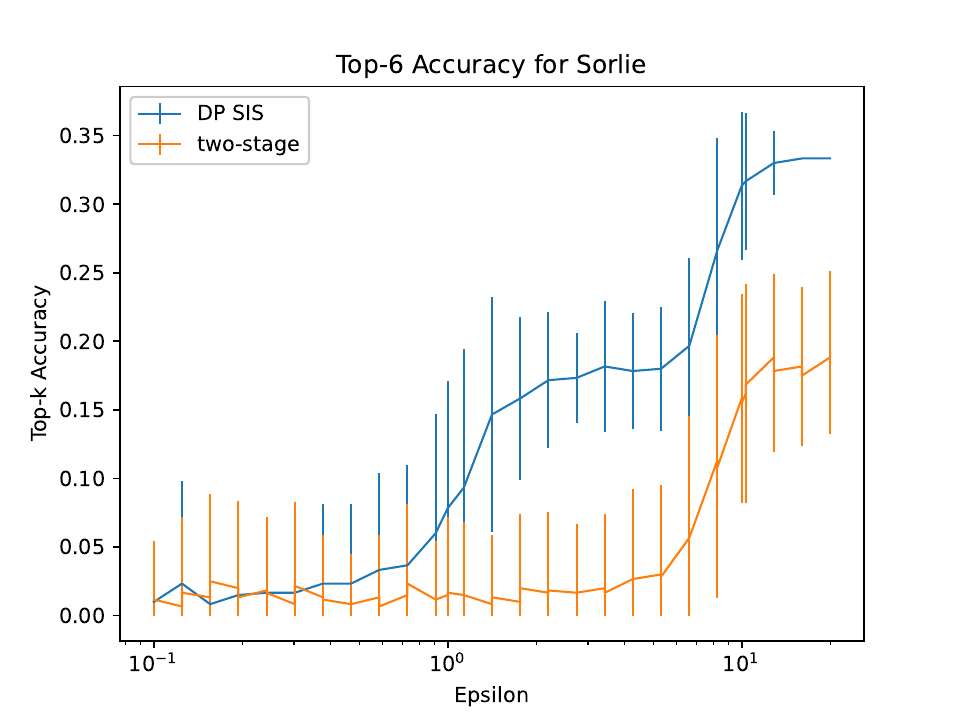}
    \end{subfigure}

    \caption{Top-$k$ accuracy of models fit on features selected from on the Sorlie dataset. DP-SIS outperforms the two-stage mechanism on $\epsilon$ values greater than $10^0$, which are commonly used for private computation \cite{near2023guidelines}.}
    \label{fig:topk_sorlie_subplots}
\end{figure}

%% file: paper_plots/yeoh/topk_yeoh_subplots.tex
\begin{figure}[t]
    \centering
    \begin{subfigure}{}
        \includegraphics[width=0.45\textwidth]{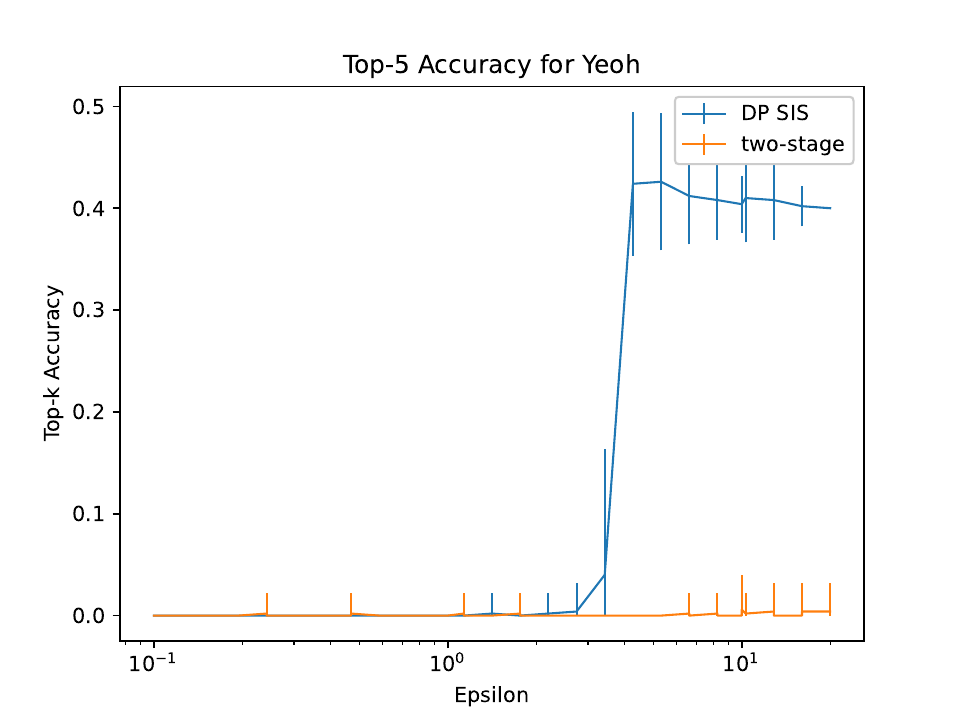}
    \end{subfigure}
    \hfill
    \begin{subfigure}{}
        \includegraphics[width=0.45\textwidth]{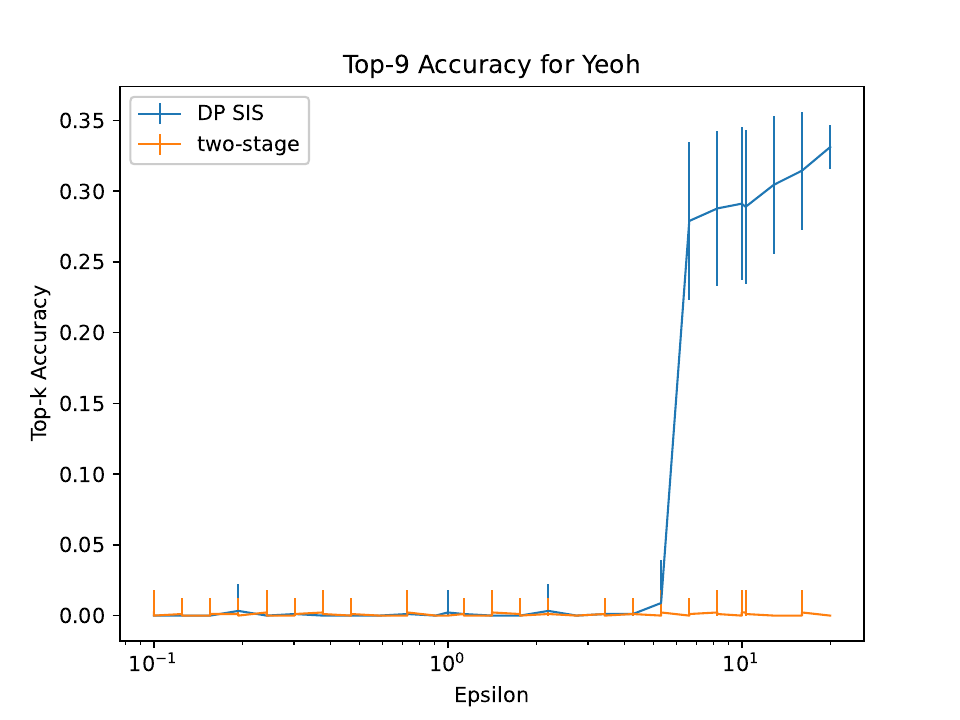}
    \end{subfigure}

    \caption{Top-$k$ accuracy of models fit on features selected from on the Yeoh dataset. DP-SIS outperforms the two-stage mechanism on $\epsilon$ values greater than $2 \times 10^0$, which can be used in private computation \cite{near2023guidelines}.}
    \label{fig:topk_yeoh_subplots}
\end{figure}

%% file: paper_plots/synth/topk_synth_subplots.tex
\begin{figure}[t]
    \centering
    \begin{subfigure}{}
        \includegraphics[width=0.45\textwidth]{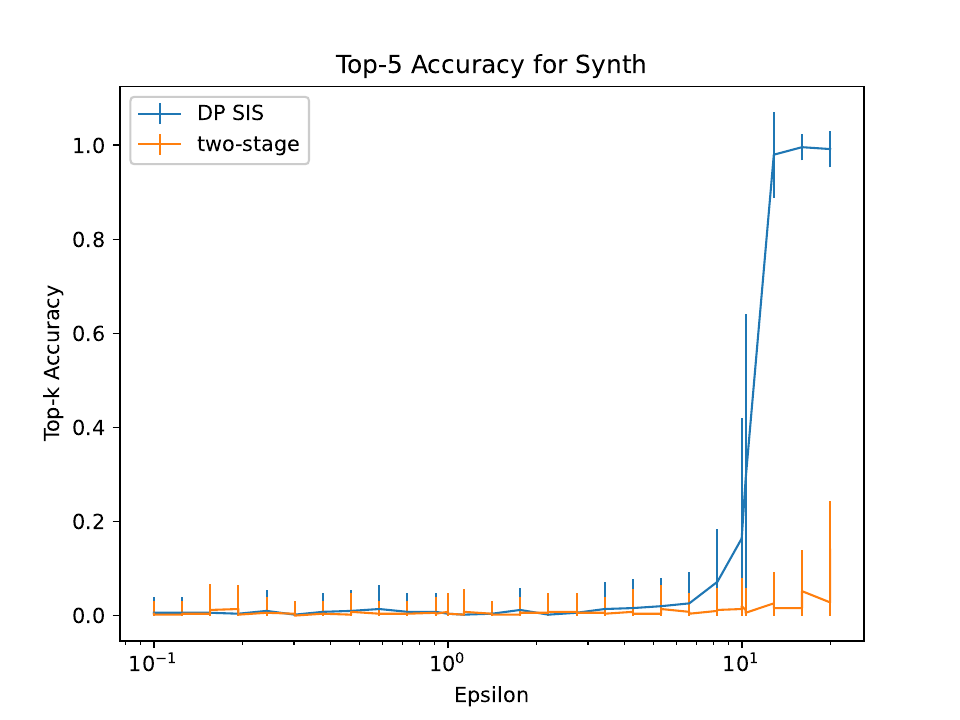}
    \end{subfigure}
    \hfill
    \begin{subfigure}{}
        \includegraphics[width=0.45\textwidth]{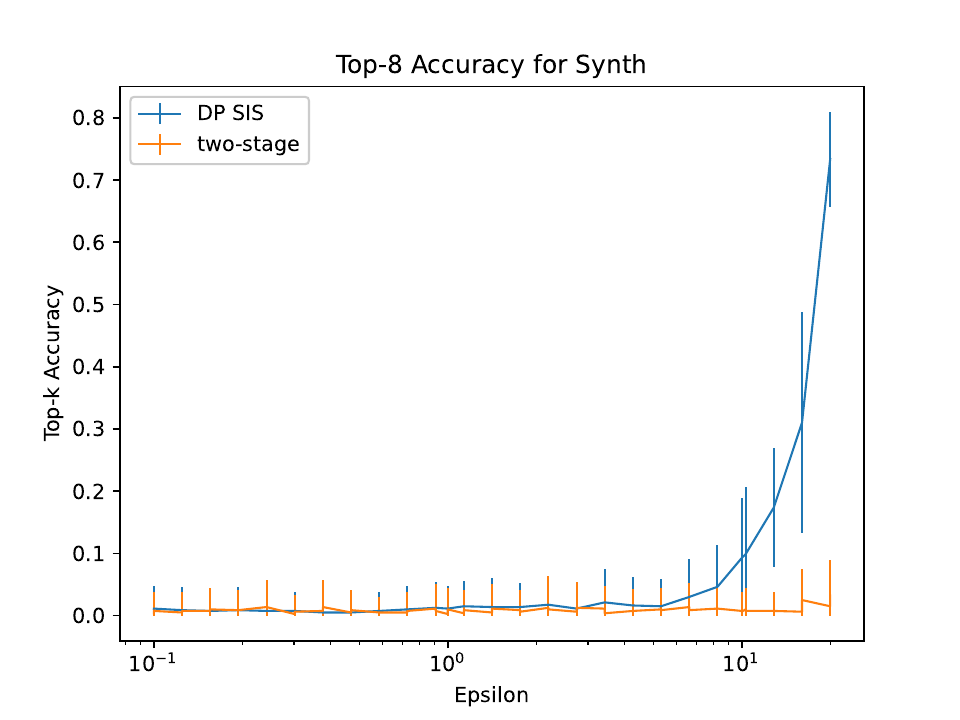}
    \end{subfigure}

    \caption{Top-$k$ accuracy of models fit on features selected from on the Synth dataset. DP-SIS outperforms the two-stage mechanism on $\epsilon$ values greater than $10^1$. Although such high $\epsilon$ values are typically not used for private computation, this result still demonstrates that DP-SIS has better results than the two-stage baseline.}
    \label{fig:topk_synth_subplots}
\end{figure}

%% file: main.bbl

\begin{thebibliography}{57}

%%% the Web address.

\ifx \showCODEN    \undefined \def \showCODEN     #1{\unskip}     \fi
\ifx \showDOI      \undefined \def \showDOI       #1{#1}\fi
\ifx \showISBNx    \undefined \def \showISBNx     #1{\unskip}     \fi
\ifx \showISBNxiii \undefined \def \showISBNxiii  #1{\unskip}     \fi
\ifx \showISSN     \undefined \def \showISSN      #1{\unskip}     \fi
\ifx \showLCCN     \undefined \def \showLCCN      #1{\unskip}     \fi
\ifx \shownote     \undefined \def \shownote      #1{#1}          \fi
\ifx \showarticletitle \undefined \def \showarticletitle #1{#1}   \fi
\ifx \showURL      \undefined \def \showURL       {\relax}        \fi
% invisible to TeX
\providecommand\bibfield[2]{#2}
\providecommand\bibinfo[2]{#2}
\providecommand\natexlab[1]{#1}
\providecommand\showeprint[2][]{arXiv:#2}

\bibitem[Alon et~al\mbox{.}(1999)]%
        {Alon1999}
\bibfield{author}{\bibinfo{person}{U. Alon}, \bibinfo{person}{N. Barkai}, \bibinfo{person}{D.A. Notterman}, \bibinfo{person}{K. Gish}, \bibinfo{person}{S. Ybarra}, \bibinfo{person}{D. Mack}, {and} \bibinfo{person}{A.J. Levine}.} \bibinfo{year}{1999}\natexlab{}.
\newblock \showarticletitle{Broad patterns of gene expression revealed by clustering analysis of tumor and normal colon tissues probed by oligonucleotide arrays}.
\newblock \bibinfo{journal}{\emph{Proceedings of the National Academy of Sciences}} \bibinfo{volume}{96}, \bibinfo{number}{12} (\bibinfo{year}{1999}), \bibinfo{pages}{6745--6750}.
\newblock


\bibitem[Borboudakis and Tsamardinos(2019)]%
        {JMLR:v20:17-334}
\bibfield{author}{\bibinfo{person}{Giorgos Borboudakis} {and} \bibinfo{person}{Ioannis Tsamardinos}.} \bibinfo{year}{2019}\natexlab{}.
\newblock \showarticletitle{Forward-Backward Selection with Early Dropping}.
\newblock \bibinfo{journal}{\emph{Journal of Machine Learning Research}} \bibinfo{volume}{20}, \bibinfo{number}{8} (\bibinfo{year}{2019}), \bibinfo{pages}{1--39}.
\newblock
\urldef\tempurl%
\url{http://jmlr.org/papers/v20/17-334.html}
\showURL{%
\tempurl}


\bibitem[Borovecki et~al\mbox{.}(2005)]%
        {Borovecki2005}
\bibfield{author}{\bibinfo{person}{F. Borovecki}, \bibinfo{person}{L. Lovrecic}, \bibinfo{person}{J. Zhou}, \bibinfo{person}{H. Jeong}, \bibinfo{person}{F. Then}, \bibinfo{person}{H.D. Rosas}, \bibinfo{person}{S.M. Hersch}, \bibinfo{person}{P. Hogarth}, \bibinfo{person}{B. Bouzou}, \bibinfo{person}{R.V. Jensen}, {and} \bibinfo{person}{D. Krainc}.} \bibinfo{year}{2005}\natexlab{}.
\newblock \showarticletitle{Genome-wide expression profiling of human blood reveals biomarkers for Huntington's disease}.
\newblock \bibinfo{journal}{\emph{Proceedings of the National Academy of Sciences}} \bibinfo{volume}{102}, \bibinfo{number}{31} (\bibinfo{year}{2005}), \bibinfo{pages}{11023--11028}.
\newblock


\bibitem[Burczynski et~al\mbox{.}(2006)]%
        {Burczynski2006}
\bibfield{author}{\bibinfo{person}{M.E. Burczynski}, \bibinfo{person}{R.L. Peterson}, \bibinfo{person}{N.C. Twine}, \bibinfo{person}{K.A. Zuberek}, \bibinfo{person}{B.J. Brodeur}, \bibinfo{person}{L. Casciotti}, \bibinfo{person}{V. Maganti}, \bibinfo{person}{H.M. Sackett}, \bibinfo{person}{N. Novoradovskaya}, \bibinfo{person}{B. Cook}, \bibinfo{person}{P.S. Reddy}, \bibinfo{person}{A. Strahs}, \bibinfo{person}{M.L. Clawson}, \bibinfo{person}{R.M. Goldschmidt}, \bibinfo{person}{S. Chaturvedi}, \bibinfo{person}{A.M. Slager}, \bibinfo{person}{L.A. Marshall}, {and} \bibinfo{person}{B. Renault}.} \bibinfo{year}{2006}\natexlab{}.
\newblock \showarticletitle{Molecular classification of Crohn's disease and ulcerative colitis patients using transcriptional profiles in peripheral blood mononuclear cells}.
\newblock \bibinfo{journal}{\emph{The Journal of Molecular Diagnostics}} \bibinfo{volume}{8}, \bibinfo{number}{1} (\bibinfo{year}{2006}), \bibinfo{pages}{51--61}.
\newblock


\bibitem[Chiaretti et~al\mbox{.}(2004)]%
        {Chiaretti2004}
\bibfield{author}{\bibinfo{person}{S. Chiaretti}, \bibinfo{person}{X. Li}, \bibinfo{person}{R. Gentleman}, \bibinfo{person}{A. Vitale}, \bibinfo{person}{M. Vignetti}, \bibinfo{person}{F. Mandelli}, \bibinfo{person}{J. Ritz}, {and} \bibinfo{person}{R. Foa}.} \bibinfo{year}{2004}\natexlab{}.
\newblock \showarticletitle{Gene expression profile of adult T-cell acute lymphocytic leukemia identifies distinct subsets of patients with different response to therapy and survival}.
\newblock \bibinfo{journal}{\emph{Blood}} \bibinfo{volume}{103}, \bibinfo{number}{7} (\bibinfo{year}{2004}), \bibinfo{pages}{2771--2778}.
\newblock


\bibitem[Chin et~al\mbox{.}(2006)]%
        {Chin2006}
\bibfield{author}{\bibinfo{person}{Koei Chin}, \bibinfo{person}{Sanjeev DeVries}, \bibinfo{person}{Jane Fridlyand}, \bibinfo{person}{Paul~T. Spellman}, \bibinfo{person}{Rajika Roydasgupta}, \bibinfo{person}{Wen-Lin Kuo}, \bibinfo{person}{Anna Lapuk}, \bibinfo{person}{Richard~M. Neve}, \bibinfo{person}{Zhen Qian}, \bibinfo{person}{Tim Ryder}, \bibinfo{person}{Nora Bayani}, \bibinfo{person}{Jonathan Brock}, \bibinfo{person}{Kimberly Montgomery}, \bibinfo{person}{David Ginzinger}, \bibinfo{person}{Dan Seah}, \bibinfo{person}{Wei Kuo}, \bibinfo{person}{Jeffrey~L. Stilwell}, \bibinfo{person}{Daniel Pinkel}, \bibinfo{person}{Donna~G. Albertson}, \bibinfo{person}{Frederic~M. Waldman}, \bibinfo{person}{Anil~N. Jain}, \bibinfo{person}{Colin Collins}, {and} \bibinfo{person}{Joe~W. Gray}.} \bibinfo{year}{2006}\natexlab{}.
\newblock \showarticletitle{Genomic and transcriptional aberrations linked to breast cancer pathophysiologies}.
\newblock \bibinfo{journal}{\emph{Cancer Cell}} \bibinfo{volume}{10}, \bibinfo{number}{6} (\bibinfo{year}{2006}), \bibinfo{pages}{529--541}.
\newblock


\bibitem[Chowdary et~al\mbox{.}(2006)]%
        {Chowdary2006}
\bibfield{author}{\bibinfo{person}{D.R. Chowdary}, \bibinfo{person}{J. Lathrop}, \bibinfo{person}{J. Skelton}, \bibinfo{person}{K. Curtin}, \bibinfo{person}{T. Briggs}, \bibinfo{person}{L. Zhang}, \bibinfo{person}{A. Rashidi}, \bibinfo{person}{S. White}, \bibinfo{person}{D. Curtis}, \bibinfo{person}{D.D. Von~Hoff}, {and} \bibinfo{person}{D. Kravitz}.} \bibinfo{year}{2006}\natexlab{}.
\newblock \showarticletitle{Prognostic gene expression signatures can be measured in tissues collected in RNAlater preservative}.
\newblock \bibinfo{journal}{\emph{The Journal of Molecular Diagnostics}} \bibinfo{volume}{8}, \bibinfo{number}{1} (\bibinfo{year}{2006}), \bibinfo{pages}{31--39}.
\newblock


\bibitem[Christensen et~al\mbox{.}(2009)]%
        {Christensen:2009gu}
\bibfield{author}{\bibinfo{person}{Brock~C Christensen}, \bibinfo{person}{E~Andres Houseman}, \bibinfo{person}{Carmen~J Marsit}, \bibinfo{person}{Shichun Zheng}, \bibinfo{person}{Margaret~R Wrensch}, \bibinfo{person}{Joseph~L Wiemels}, \bibinfo{person}{Heather~H Nelson}, \bibinfo{person}{Margaret~R Karagas}, \bibinfo{person}{James~F Padbury}, \bibinfo{person}{Raphael Bueno}, \bibinfo{person}{David~J Sugarbaker}, \bibinfo{person}{Ru-Fang Yeh}, \bibinfo{person}{John~K Wiencke}, {and} \bibinfo{person}{Karl~T Kelsey}.} \bibinfo{year}{2009}\natexlab{}.
\newblock \showarticletitle{{Aging and Environmental Exposures Alter Tissue-Specific DNA Methylation Dependent upon CpG Island Context}}.
\newblock \bibinfo{journal}{\emph{PLOS Genetics}} \bibinfo{volume}{5}, \bibinfo{number}{8} (\bibinfo{date}{Aug.} \bibinfo{year}{2009}), \bibinfo{pages}{e1000602}.
\newblock


\bibitem[Donoho and Elad(2003)]%
        {donoho2003maximal}
\bibfield{author}{\bibinfo{person}{DL Donoho} {and} \bibinfo{person}{M Elad}.} \bibinfo{year}{2003}\natexlab{}.
\newblock \showarticletitle{Maximal sparsity representation via L1 minimization}.
\newblock \bibinfo{journal}{\emph{Proceedings of National Academy of Sciences}}  \bibinfo{volume}{100} (\bibinfo{year}{2003}), \bibinfo{pages}{2197--2202}.
\newblock


\bibitem[Donoho et~al\mbox{.}(2001)]%
        {donoho2001uncertainty}
\bibfield{author}{\bibinfo{person}{David~L Donoho}, \bibinfo{person}{Xiaoming Huo}, {et~al\mbox{.}}} \bibinfo{year}{2001}\natexlab{}.
\newblock \showarticletitle{Uncertainty principles and ideal atomic decomposition}.
\newblock \bibinfo{journal}{\emph{IEEE transactions on information theory}} \bibinfo{volume}{47}, \bibinfo{number}{7} (\bibinfo{year}{2001}), \bibinfo{pages}{2845--2862}.
\newblock


\bibitem[Durfee and Rogers(2019)]%
        {durfee2019practical}
\bibfield{author}{\bibinfo{person}{David Durfee} {and} \bibinfo{person}{Ryan~M Rogers}.} \bibinfo{year}{2019}\natexlab{}.
\newblock \showarticletitle{Practical differentially private top-k selection with pay-what-you-get composition}.
\newblock \bibinfo{journal}{\emph{Advances in Neural Information Processing Systems}}  \bibinfo{volume}{32} (\bibinfo{year}{2019}).
\newblock


\bibitem[Fan and Lv(2006)]%
        {Fan2006SureIS}
\bibfield{author}{\bibinfo{person}{Jianqing Fan} {and} \bibinfo{person}{Jinchi Lv}.} \bibinfo{year}{2006}\natexlab{}.
\newblock \showarticletitle{Sure independence screening for ultrahigh dimensional feature space}.
\newblock \bibinfo{journal}{\emph{Journal of the Royal Statistical Society: Series B (Statistical Methodology)}}  \bibinfo{volume}{70} (\bibinfo{year}{2006}).
\newblock
\urldef\tempurl%
\url{https://api.semanticscholar.org/CorpusID:5001358}
\showURL{%
\tempurl}


\bibitem[Fan and Lv(2008)]%
        {fan2008sure}
\bibfield{author}{\bibinfo{person}{Jianqing Fan} {and} \bibinfo{person}{Jinchi Lv}.} \bibinfo{year}{2008}\natexlab{}.
\newblock \showarticletitle{Sure independence screening for ultrahigh dimensional feature space}.
\newblock \bibinfo{journal}{\emph{Journal of the Royal Statistical Society Series B: Statistical Methodology}} \bibinfo{volume}{70}, \bibinfo{number}{5} (\bibinfo{year}{2008}), \bibinfo{pages}{849--911}.
\newblock


\bibitem[Fan et~al\mbox{.}(2008)]%
        {10.5555/1390681.1442794}
\bibfield{author}{\bibinfo{person}{Rong-En Fan}, \bibinfo{person}{Kai-Wei Chang}, \bibinfo{person}{Cho-Jui Hsieh}, \bibinfo{person}{Xiang-Rui Wang}, {and} \bibinfo{person}{Chih-Jen Lin}.} \bibinfo{year}{2008}\natexlab{}.
\newblock \showarticletitle{LIBLINEAR: A Library for Large Linear Classification}.
\newblock \bibinfo{journal}{\emph{J. Mach. Learn. Res.}}  \bibinfo{volume}{9} (\bibinfo{date}{jun} \bibinfo{year}{2008}), \bibinfo{pages}{1871–1874}.
\newblock
\showISSN{1532-4435}


\bibitem[Frandi et~al\mbox{.}(2016)]%
        {frandi_fast_2016}
\bibfield{author}{\bibinfo{person}{Emanuele Frandi}, \bibinfo{person}{Ricardo Ñanculef}, \bibinfo{person}{Stefano Lodi}, \bibinfo{person}{Claudio Sartori}, {and} \bibinfo{person}{Johan A.~K. Suykens}.} \bibinfo{year}{2016}\natexlab{}.
\newblock \showarticletitle{Fast and scalable {Lasso} via stochastic {Frank}–{Wolfe} methods with a convergence guarantee}.
\newblock \bibinfo{journal}{\emph{Machine Learning}} \bibinfo{volume}{104}, \bibinfo{number}{2} (\bibinfo{date}{Sept.} \bibinfo{year}{2016}), \bibinfo{pages}{195--221}.
\newblock
\showISSN{1573-0565}
\urldef\tempurl%
\url{https://doi.org/10.1007/s10994-016-5578-4}
\showDOI{\tempurl}


\bibitem[Friedman et~al\mbox{.}(2010)]%
        {Friedman2010}
\bibfield{author}{\bibinfo{person}{Jerome Friedman}, \bibinfo{person}{Trevor Hastie}, {and} \bibinfo{person}{Rob Tibshirani}.} \bibinfo{year}{2010}\natexlab{}.
\newblock \showarticletitle{Regularization {Paths} for {Generalized} {Linear} {Models} via {Coordinate} {Descent}}.
\newblock \bibinfo{journal}{\emph{Journal of Statistical Software}} \bibinfo{volume}{33}, \bibinfo{number}{1} (\bibinfo{year}{2010}), \bibinfo{pages}{1--22}.
\newblock
\showISSN{19390068}
\newblock
\shownote{arXiv: 1501.0228 ISBN: 9781439811870}.


\bibitem[Gillenwater et~al\mbox{.}(2022)]%
        {gillenwater2022joint}
\bibfield{author}{\bibinfo{person}{Jennifer Gillenwater}, \bibinfo{person}{Matthew Joseph}, \bibinfo{person}{Andres Munoz}, {and} \bibinfo{person}{Monica~Ribero Diaz}.} \bibinfo{year}{2022}\natexlab{}.
\newblock \showarticletitle{A Joint Exponential Mechanism For Differentially Private Top-$ k$}. In \bibinfo{booktitle}{\emph{International Conference on Machine Learning}}. PMLR, \bibinfo{pages}{7570--7582}.
\newblock


\bibitem[Golub et~al\mbox{.}(1999)]%
        {Golub1999}
\bibfield{author}{\bibinfo{person}{Todd~R. Golub}, \bibinfo{person}{Donna~K. Slonim}, \bibinfo{person}{Pablo Tamayo}, \bibinfo{person}{Christine Huard}, \bibinfo{person}{Michelle Gaasenbeek}, \bibinfo{person}{Jill~P. Mesirov}, \bibinfo{person}{Hilary Coller}, \bibinfo{person}{Mignon~L. Loh}, \bibinfo{person}{James~R. Downing}, \bibinfo{person}{Michael~A. Caligiuri}, \bibinfo{person}{Clara~D. Bloomfield}, {and} \bibinfo{person}{Eric~S. Lander}.} \bibinfo{year}{1999}\natexlab{}.
\newblock \showarticletitle{Molecular classification of cancer: class discovery and class prediction by gene expression monitoring}.
\newblock \bibinfo{journal}{\emph{Science}} \bibinfo{volume}{286}, \bibinfo{number}{5439} (\bibinfo{year}{1999}), \bibinfo{pages}{531--537}.
\newblock


\bibitem[Gordon et~al\mbox{.}(2002)]%
        {Gordon2002}
\bibfield{author}{\bibinfo{person}{G.J. Gordon}, \bibinfo{person}{R.V. Jensen}, \bibinfo{person}{L.L. Hsiao}, \bibinfo{person}{S.R. Gullans}, \bibinfo{person}{J.E. Blumenstock}, \bibinfo{person}{S. Ramaswamy}, \bibinfo{person}{W.G. Richards}, \bibinfo{person}{D.J. Sugarbaker}, {and} \bibinfo{person}{R. Bueno}.} \bibinfo{year}{2002}\natexlab{}.
\newblock \showarticletitle{Translation of microarray data into clinically relevant cancer diagnostic tests using gene expression ratios in lung cancer and mesothelioma}.
\newblock \bibinfo{journal}{\emph{Cancer Research}} \bibinfo{volume}{62}, \bibinfo{number}{17} (\bibinfo{year}{2002}), \bibinfo{pages}{4963--4967}.
\newblock


\bibitem[{Gravier, Eleonore} et~al\mbox{.}(2010)]%
        {Gravier:2010bz}
\bibfield{author}{\bibinfo{person}{{Gravier, Eleonore}}, \bibinfo{person}{Gaelle Pierron}, \bibinfo{person}{Anne Vincent-Salomon}, \bibinfo{person}{Nadege gruel}, \bibinfo{person}{Virginie Raynal}, \bibinfo{person}{Alexia Savignoni}, \bibinfo{person}{Yann De~Rycke}, \bibinfo{person}{Jean-Yves Pierga}, \bibinfo{person}{Carlo Lucchesi}, \bibinfo{person}{Fabien Reyal}, \bibinfo{person}{Alain Fourquet}, \bibinfo{person}{Sergio Roman-Roman}, \bibinfo{person}{Francois Radvanyi}, \bibinfo{person}{Xavier Sastre-Garau}, \bibinfo{person}{Bernard Asselain}, {and} \bibinfo{person}{Olivier Delattre}.} \bibinfo{year}{2010}\natexlab{}.
\newblock \showarticletitle{{A prognostic DNA signature for T1T2 node-negative breast cancer patients.}}
\newblock \bibinfo{journal}{\emph{Genes, Chromosomes and Cancer}} \bibinfo{volume}{49}, \bibinfo{number}{12} (\bibinfo{date}{Sept.} \bibinfo{year}{2010}), \bibinfo{pages}{1125--1125}.
\newblock


\bibitem[Hastie et~al\mbox{.}(2020)]%
        {hastie_best_2020}
\bibfield{author}{\bibinfo{person}{Trevor Hastie}, \bibinfo{person}{Robert Tibshirani}, {and} \bibinfo{person}{Ryan Tibshirani}.} \bibinfo{year}{2020}\natexlab{}.
\newblock \showarticletitle{Best {Subset}, {Forward} {Stepwise} or {Lasso}? {Analysis} and {Recommendations} {Based} on {Extensive} {Comparisons}}.
\newblock \bibinfo{journal}{\emph{Statist. Sci.}} \bibinfo{volume}{35}, \bibinfo{number}{4} (\bibinfo{date}{Nov.} \bibinfo{year}{2020}), \bibinfo{pages}{579--592}.
\newblock
\showISSN{0883-4237, 2168-8745}
\urldef\tempurl%
\url{https://doi.org/10.1214/19-STS733}
\showDOI{\tempurl}
\newblock
\shownote{Publisher: Institute of Mathematical Statistics}.


\bibitem[Hu et~al\mbox{.}(2024)]%
        {hu2024provable}
\bibfield{author}{\bibinfo{person}{Yaxi Hu}, \bibinfo{person}{Amartya Sanyal}, {and} \bibinfo{person}{Bernhard Sch{\"o}lkopf}.} \bibinfo{year}{2024}\natexlab{}.
\newblock \showarticletitle{Provable Privacy with Non-Private Pre-Processing}.
\newblock \bibinfo{journal}{\emph{arXiv preprint arXiv:2403.13041}} (\bibinfo{year}{2024}).
\newblock


\bibitem[Khan et~al\mbox{.}(2001)]%
        {Khan2001}
\bibfield{author}{\bibinfo{person}{J. Khan}, \bibinfo{person}{J.S. Wei}, \bibinfo{person}{M. Ringner}, \bibinfo{person}{L.H. Saal}, \bibinfo{person}{M. Ladanyi}, \bibinfo{person}{F. Westermann}, \bibinfo{person}{F. Berthold}, \bibinfo{person}{M. Schwab}, \bibinfo{person}{C.R. Antonescu}, \bibinfo{person}{C. Peterson}, {and} \bibinfo{person}{P.S. Meltzer}.} \bibinfo{year}{2001}\natexlab{}.
\newblock \showarticletitle{Classification and diagnostic prediction of cancers using gene expression profiling and artificial neural networks}.
\newblock \bibinfo{journal}{\emph{Nature Medicine}} \bibinfo{volume}{7}, \bibinfo{number}{6} (\bibinfo{year}{2001}), \bibinfo{pages}{673--679}.
\newblock


\bibitem[Khanna et~al\mbox{.}(2023a)]%
        {khanna2023challenge}
\bibfield{author}{\bibinfo{person}{Amol Khanna}, \bibinfo{person}{Fred Lu}, {and} \bibinfo{person}{Edward Raff}.} \bibinfo{year}{2023}\natexlab{a}.
\newblock \showarticletitle{The challenge of differentially private screening rules}.
\newblock \bibinfo{journal}{\emph{arXiv preprint arXiv:2303.10303}} (\bibinfo{year}{2023}).
\newblock


\bibitem[Khanna et~al\mbox{.}(2023b)]%
        {khanna2023differentially}
\bibfield{author}{\bibinfo{person}{Amol Khanna}, \bibinfo{person}{Fred Lu}, \bibinfo{person}{Edward Raff}, {and} \bibinfo{person}{Brian Testa}.} \bibinfo{year}{2023}\natexlab{b}.
\newblock \showarticletitle{Differentially Private Logistic Regression with Sparse Solutions}. In \bibinfo{booktitle}{\emph{Proceedings of the 16th ACM Workshop on Artificial Intelligence and Security}}. \bibinfo{pages}{1--9}.
\newblock


\bibitem[Khanna et~al\mbox{.}(2024)]%
        {khanna2024sok}
\bibfield{author}{\bibinfo{person}{Amol Khanna}, \bibinfo{person}{Edward Raff}, {and} \bibinfo{person}{Nathan Inkawhich}.} \bibinfo{year}{2024}\natexlab{}.
\newblock \showarticletitle{SoK: A Review of Differentially Private Linear Models For High-Dimensional Data}. In \bibinfo{booktitle}{\emph{2024 IEEE Conference on Secure and Trustworthy Machine Learning (SaTML)}}. IEEE, \bibinfo{pages}{57--77}.
\newblock


\bibitem[Kifer et~al\mbox{.}(2012)]%
        {kifer2012private}
\bibfield{author}{\bibinfo{person}{Daniel Kifer}, \bibinfo{person}{Adam Smith}, {and} \bibinfo{person}{Abhradeep Thakurta}.} \bibinfo{year}{2012}\natexlab{}.
\newblock \showarticletitle{Private convex empirical risk minimization and high-dimensional regression}. In \bibinfo{booktitle}{\emph{Conference on Learning Theory}}. JMLR Workshop and Conference Proceedings, \bibinfo{pages}{25--1}.
\newblock


\bibitem[Larsson(2021)]%
        {Larsson2021}
\bibfield{author}{\bibinfo{person}{Johan Larsson}.} \bibinfo{year}{2021}\natexlab{}.
\newblock \showarticletitle{Look-{Ahead} {Screening} {Rules} for the {Lasso}}.
\newblock  (\bibinfo{year}{2021}).
\newblock
\urldef\tempurl%
\url{http://arxiv.org/abs/2105.05648}
\showURL{%
\tempurl}
\newblock
\shownote{arXiv: 2105.05648}.


\bibitem[Liu and Nocedal(1989)]%
        {Liu1989}
\bibfield{author}{\bibinfo{person}{Dong~C Liu} {and} \bibinfo{person}{Jorge Nocedal}.} \bibinfo{year}{1989}\natexlab{}.
\newblock \showarticletitle{On the limited memory {BFGS} method for large scale optimization}.
\newblock \bibinfo{journal}{\emph{Mathematical Programming}} \bibinfo{volume}{45}, \bibinfo{number}{1} (\bibinfo{year}{1989}), \bibinfo{pages}{503--528}.
\newblock
\showISSN{1436-4646}
\urldef\tempurl%
\url{https://doi.org/10.1007/BF01589116}
\showDOI{\tempurl}


\bibitem[Ndiaye et~al\mbox{.}(2017)]%
        {ndiaye_gap_2017}
\bibfield{author}{\bibinfo{person}{Eugene Ndiaye}, \bibinfo{person}{Olivier Fercoq}, \bibinfo{person}{Alex}, \bibinfo{person}{Re Gramfort}, {and} \bibinfo{person}{Joseph Salmon}.} \bibinfo{year}{2017}\natexlab{}.
\newblock \showarticletitle{Gap {Safe} {Screening} {Rules} for {Sparsity} {Enforcing} {Penalties}}.
\newblock \bibinfo{journal}{\emph{Journal of Machine Learning Research}} \bibinfo{volume}{18}, \bibinfo{number}{128} (\bibinfo{year}{2017}), \bibinfo{pages}{1--33}.
\newblock
\showISSN{1533-7928}
\urldef\tempurl%
\url{http://jmlr.org/papers/v18/16-577.html}
\showURL{%
\tempurl}


\bibitem[Near and Abuah(2021)]%
        {near2021programming}
\bibfield{author}{\bibinfo{person}{Joseph~P Near} {and} \bibinfo{person}{Chik{\'e} Abuah}.} \bibinfo{year}{2021}\natexlab{}.
\newblock \showarticletitle{Programming differential privacy}.
\newblock \bibinfo{journal}{\emph{URL: https://uvm}} (\bibinfo{year}{2021}).
\newblock


\bibitem[Near et~al\mbox{.}(2023)]%
        {near2023guidelines}
\bibfield{author}{\bibinfo{person}{Joseph~P Near}, \bibinfo{person}{David Darais}, \bibinfo{person}{Naomi Lefkovitz}, \bibinfo{person}{Gary Howarth}, {et~al\mbox{.}}} \bibinfo{year}{2023}\natexlab{}.
\newblock \bibinfo{booktitle}{\emph{Guidelines for Evaluating Differential Privacy Guarantees}}.
\newblock \bibinfo{type}{{T}echnical {R}eport}. \bibinfo{institution}{National Institute of Standards and Technology}.
\newblock


\bibitem[Ng(2004)]%
        {Ng2004}
\bibfield{author}{\bibinfo{person}{Andrew~Y. Ng}.} \bibinfo{year}{2004}\natexlab{}.
\newblock \showarticletitle{Feature selection, {L1} vs. {L2} regularization, and rotational invariance}.
\newblock \bibinfo{journal}{\emph{Twenty-first international conference on Machine learning - ICML '04}} (\bibinfo{year}{2004}), \bibinfo{pages}{78}.
\newblock
\urldef\tempurl%
\url{https://doi.org/10.1145/1015330.1015435}
\showDOI{\tempurl}
\newblock
\shownote{Publisher: ACM Press Place: New York, New York, USA ISBN: 1581138285}.


\bibitem[Ogawa et~al\mbox{.}(2013)]%
        {Ogawa2013}
\bibfield{author}{\bibinfo{person}{Kohei Ogawa}, \bibinfo{person}{Yoshiki Suzuki}, {and} \bibinfo{person}{Ichiro Takeuchi}.} \bibinfo{year}{2013}\natexlab{}.
\newblock \showarticletitle{Safe screening of non-support vectors in pathwise {SVM} computation}.
\newblock \bibinfo{journal}{\emph{ICML}} (\bibinfo{year}{2013}).
\newblock
\urldef\tempurl%
\url{http://jmlr.org/proceedings/papers/v28/ogawa13b.html}
\showURL{%
\tempurl}


\bibitem[Pedregosa et~al\mbox{.}(2011)]%
        {scikit-learn}
\bibfield{author}{\bibinfo{person}{F Pedregosa}, \bibinfo{person}{G Varoquaux}, \bibinfo{person}{A Gramfort}, \bibinfo{person}{V Michel}, \bibinfo{person}{B Thirion}, \bibinfo{person}{O Grisel}, \bibinfo{person}{M Blondel}, \bibinfo{person}{P Prettenhofer}, \bibinfo{person}{R Weiss}, \bibinfo{person}{V Dubourg}, \bibinfo{person}{J Vanderplas}, \bibinfo{person}{A Passos}, \bibinfo{person}{D Cournapeau}, \bibinfo{person}{M Brucher}, \bibinfo{person}{M Perrot}, {and} \bibinfo{person}{E Duchesnay}.} \bibinfo{year}{2011}\natexlab{}.
\newblock \showarticletitle{Scikit-learn: {Machine} {Learning} in {Python}}.
\newblock \bibinfo{journal}{\emph{Journal of Machine Learning Research}}  \bibinfo{volume}{12} (\bibinfo{year}{2011}), \bibinfo{pages}{2825--2830}.
\newblock
\urldef\tempurl%
\url{http://jmlr.csail.mit.edu/papers/v12/pedregosa11a.html}
\showURL{%
\tempurl}


\bibitem[Qiao et~al\mbox{.}(2021)]%
        {qiao2021oneshot}
\bibfield{author}{\bibinfo{person}{Gang Qiao}, \bibinfo{person}{Weijie Su}, {and} \bibinfo{person}{Li Zhang}.} \bibinfo{year}{2021}\natexlab{}.
\newblock \showarticletitle{Oneshot differentially private top-k selection}. In \bibinfo{booktitle}{\emph{International Conference on Machine Learning}}. PMLR, \bibinfo{pages}{8672--8681}.
\newblock


\bibitem[Raff et~al\mbox{.}(2024)]%
        {raff2024scaling}
\bibfield{author}{\bibinfo{person}{Edward Raff}, \bibinfo{person}{Amol Khanna}, {and} \bibinfo{person}{Fred Lu}.} \bibinfo{year}{2024}\natexlab{}.
\newblock \showarticletitle{Scaling Up Differentially Private LASSO Regularized Logistic Regression via Faster Frank-Wolfe Iterations}.
\newblock \bibinfo{journal}{\emph{Advances in Neural Information Processing Systems}}  \bibinfo{volume}{36} (\bibinfo{year}{2024}).
\newblock


\bibitem[Rakotomamonjy et~al\mbox{.}(2019)]%
        {rakotomamonjy_screening_2019}
\bibfield{author}{\bibinfo{person}{Alain Rakotomamonjy}, \bibinfo{person}{Gilles Gasso}, {and} \bibinfo{person}{Joseph Salmon}.} \bibinfo{year}{2019}\natexlab{}.
\newblock \showarticletitle{Screening rules for {Lasso} with non-convex {Sparse} {Regularizers}}. In \bibinfo{booktitle}{\emph{Proceedings of the 36th {International} {Conference} on {Machine} {Learning}}}. \bibinfo{publisher}{PMLR}, \bibinfo{pages}{5341--5350}.
\newblock
\urldef\tempurl%
\url{https://proceedings.mlr.press/v97/rakotomamonjy19a.html}
\showURL{%
\tempurl}
\newblock
\shownote{ISSN: 2640-3498}.


\bibitem[Ross(2014)]%
        {Ross2014}
\bibfield{author}{\bibinfo{person}{Brian~C. Ross}.} \bibinfo{year}{2014}\natexlab{}.
\newblock \showarticletitle{Mutual information between discrete and continuous data sets}.
\newblock \bibinfo{journal}{\emph{PLoS ONE}} \bibinfo{volume}{9}, \bibinfo{number}{2} (\bibinfo{year}{2014}).
\newblock
\showISSN{19326203}
\urldef\tempurl%
\url{https://doi.org/10.1371/journal.pone.0087357}
\showDOI{\tempurl}


\bibitem[Roy and Tewari(2023)]%
        {roy2023computational}
\bibfield{author}{\bibinfo{person}{Saptarshi Roy} {and} \bibinfo{person}{Ambuj Tewari}.} \bibinfo{year}{2023}\natexlab{}.
\newblock \showarticletitle{On the Computational Complexity of Private High-dimensional Model Selection via the Exponential Mechanism}.
\newblock \bibinfo{journal}{\emph{arXiv preprint arXiv:2310.07852}} (\bibinfo{year}{2023}).
\newblock


\bibitem[Shekelyan and Loukides(2022)]%
        {shekelyan2022differentially}
\bibfield{author}{\bibinfo{person}{Michael Shekelyan} {and} \bibinfo{person}{Grigorios Loukides}.} \bibinfo{year}{2022}\natexlab{}.
\newblock \showarticletitle{Differentially Private Top-k Selection via Canonical Lipschitz Mechanism}.
\newblock \bibinfo{journal}{\emph{arXiv preprint arXiv:2201.13376}} (\bibinfo{year}{2022}).
\newblock


\bibitem[Shipp et~al\mbox{.}(2002)]%
        {Shipp2002}
\bibfield{author}{\bibinfo{person}{Margaret~A. Shipp}, \bibinfo{person}{Kenneth~N. Ross}, \bibinfo{person}{Pablo Tamayo}, \bibinfo{person}{Alex~P. Weng}, \bibinfo{person}{Jeffery~L. Kutok}, \bibinfo{person}{Ricardo C.~T. Aguiar}, \bibinfo{person}{Michelle Gaasenbeek}, \bibinfo{person}{Maria Angelo}, \bibinfo{person}{Margaret Reich}, \bibinfo{person}{Geraldine~S. Pinkus}, \bibinfo{person}{Thomas~S. Ray}, \bibinfo{person}{Michael~A. Koval}, \bibinfo{person}{Kenneth~W. Last}, \bibinfo{person}{Andrew Norton}, \bibinfo{person}{T.~Andrew Lister}, \bibinfo{person}{Jill Mesirov}, \bibinfo{person}{Donna~S. Neuberg}, \bibinfo{person}{Eric~S. Lander}, \bibinfo{person}{Jon~C. Aster}, {and} \bibinfo{person}{Todd~R. Golub}.} \bibinfo{year}{2002}\natexlab{}.
\newblock \showarticletitle{Diffuse large B-cell lymphoma outcome prediction by gene-expression profiling and supervised machine learning}.
\newblock \bibinfo{journal}{\emph{Nature Medicine}} \bibinfo{volume}{8}, \bibinfo{number}{1} (\bibinfo{year}{2002}), \bibinfo{pages}{68--74}.
\newblock


\bibitem[Singh et~al\mbox{.}(2002)]%
        {Singh2002}
\bibfield{author}{\bibinfo{person}{Dinesh Singh}, \bibinfo{person}{Phillip~G. Febbo}, \bibinfo{person}{Kenneth Ross}, \bibinfo{person}{Douglas~G. Jackson}, \bibinfo{person}{Judith Manola}, \bibinfo{person}{Carol Ladd}, \bibinfo{person}{Pablo Tamayo}, \bibinfo{person}{Andrew~A. Renshaw}, \bibinfo{person}{Anthony~V. D’Amico}, \bibinfo{person}{Jerome~P. Richie}, \bibinfo{person}{Eric~S. Lander}, \bibinfo{person}{Massimo Loda}, \bibinfo{person}{Philip~W. Kantoff}, \bibinfo{person}{Todd~R. Golub}, {and} \bibinfo{person}{William~R. Sellers}.} \bibinfo{year}{2002}\natexlab{}.
\newblock \showarticletitle{Gene expression correlates of clinical prostate cancer behavior}.
\newblock \bibinfo{journal}{\emph{Cancer Cell}} \bibinfo{volume}{1}, \bibinfo{number}{2} (\bibinfo{year}{2002}), \bibinfo{pages}{203--209}.
\newblock


\bibitem[S{\o}rlie et~al\mbox{.}(2001)]%
        {Sorlie:2001kr}
\bibfield{author}{\bibinfo{person}{Therese S{\o}rlie}, \bibinfo{person}{Charles~M Perou}, \bibinfo{person}{Robert Tibshirani}, \bibinfo{person}{Turid Aas}, \bibinfo{person}{Stephanie Geisler}, \bibinfo{person}{Hilde Johnsen}, \bibinfo{person}{Trevor Hastie}, \bibinfo{person}{Michael~B Eisen}, \bibinfo{person}{Matt van~de Rijn}, \bibinfo{person}{Stefanie~S Jeffrey}, \bibinfo{person}{Thor Thorsen}, \bibinfo{person}{Hanne Quist}, \bibinfo{person}{John~C Matese}, \bibinfo{person}{Patrick~O Brown}, \bibinfo{person}{David Botstein}, \bibinfo{person}{Per~Eystein L{\o}nning}, {and} \bibinfo{person}{Anne-Lise B{\o}rresen-Dale}.} \bibinfo{year}{2001}\natexlab{}.
\newblock \showarticletitle{{Gene expression patterns of breast carcinomas distinguish tumor subclasses with clinical implications}}.
\newblock \bibinfo{journal}{\emph{Proceedings of the National Academy of Sciences}}  \bibinfo{volume}{98} (\bibinfo{date}{Sept.} \bibinfo{year}{2001}), \bibinfo{pages}{10869--10874}.
\newblock


\bibitem[Subramanian et~al\mbox{.}(2005)]%
        {Subramanian2005}
\bibfield{author}{\bibinfo{person}{Aravind Subramanian}, \bibinfo{person}{Pablo Tamayo}, \bibinfo{person}{Vamsi~K. Mootha}, \bibinfo{person}{Sayan Mukherjee}, \bibinfo{person}{Benjamin~L. Ebert}, \bibinfo{person}{Michael~A. Gillette}, \bibinfo{person}{Amanda Paulovich}, \bibinfo{person}{Scott~L. Pomeroy}, \bibinfo{person}{Todd~R. Golub}, \bibinfo{person}{Eric~S. Lander}, {and} \bibinfo{person}{Jill~P. Mesirov}.} \bibinfo{year}{2005}\natexlab{}.
\newblock \showarticletitle{Gene set enrichment analysis: A knowledge-based approach for interpreting genome-wide expression profiles}.
\newblock \bibinfo{journal}{\emph{Proceedings of the National Academy of Sciences}} \bibinfo{volume}{102}, \bibinfo{number}{43} (\bibinfo{year}{2005}), \bibinfo{pages}{15545--15550}.
\newblock


\bibitem[Thakurta and Smith(2013)]%
        {thakurta2013differentially}
\bibfield{author}{\bibinfo{person}{Abhradeep~Guha Thakurta} {and} \bibinfo{person}{Adam Smith}.} \bibinfo{year}{2013}\natexlab{}.
\newblock \showarticletitle{Differentially private feature selection via stability arguments, and the robustness of the lasso}. In \bibinfo{booktitle}{\emph{Conference on Learning Theory}}. PMLR, \bibinfo{pages}{819--850}.
\newblock


\bibitem[Tian et~al\mbox{.}(2003)]%
        {Tian2003}
\bibfield{author}{\bibinfo{person}{Eugene Tian}, \bibinfo{person}{James~R. Sawyer}, \bibinfo{person}{Amelia~H. Ligon}, \bibinfo{person}{Anand~S. Lagoo}, \bibinfo{person}{Sherryl~L. Hubbard}, \bibinfo{person}{Kathy~L. Myers}, \bibinfo{person}{Susan~G. Hilsenbeck}, \bibinfo{person}{Ronald~J. Berenson}, \bibinfo{person}{David~O. Dixon}, \bibinfo{person}{Jennifer~R. Sawyer}, \bibinfo{person}{Bart Barlogie}, {and} \bibinfo{person}{John~D. Shaughnessy}.} \bibinfo{year}{2003}\natexlab{}.
\newblock \showarticletitle{High-resolution fluorescence in situ hybridization mapping of recurrent breakpoint regions in multiple myeloma translocations}.
\newblock \bibinfo{journal}{\emph{Cancer Research}} \bibinfo{volume}{63}, \bibinfo{number}{2} (\bibinfo{year}{2003}), \bibinfo{pages}{532--539}.
\newblock


\bibitem[Tibshirani(1994)]%
        {Tibshirani1994}
\bibfield{author}{\bibinfo{person}{Robert Tibshirani}.} \bibinfo{year}{1994}\natexlab{}.
\newblock \showarticletitle{Regression {Shrinkage} and {Selection} {Via} the {Lasso}}.
\newblock \bibinfo{journal}{\emph{Journal of the Royal Statistical Society, Series B}} \bibinfo{volume}{58}, \bibinfo{number}{1} (\bibinfo{year}{1994}), \bibinfo{pages}{267--288}.
\newblock


\bibitem[Wainwright(2019)]%
        {wainwright2019high}
\bibfield{author}{\bibinfo{person}{Martin~J Wainwright}.} \bibinfo{year}{2019}\natexlab{}.
\newblock \bibinfo{booktitle}{\emph{High-dimensional statistics: A non-asymptotic viewpoint}}. Vol.~\bibinfo{volume}{48}.
\newblock \bibinfo{publisher}{Cambridge university press}.
\newblock


\bibitem[Wang and Ding(2019)]%
        {wang2019fast}
\bibfield{author}{\bibinfo{person}{Chi Wang} {and} \bibinfo{person}{Bailu Ding}.} \bibinfo{year}{2019}\natexlab{}.
\newblock \showarticletitle{Fast {Approximation} of {Empirical} {Entropy} via {Subsampling}}. In \bibinfo{booktitle}{\emph{{25TH} {ACM} {SIGKDD} {CONFERENCE} {ON} {KNOWLEDGE} {DISCOVERY} {AND} {DATA} {MINING}}}.
\newblock


\bibitem[Wang et~al\mbox{.}(2013)]%
        {wang_lasso_2013}
\bibfield{author}{\bibinfo{person}{Jie Wang}, \bibinfo{person}{Jiayu Zhou}, \bibinfo{person}{Peter Wonka}, {and} \bibinfo{person}{Jieping Ye}.} \bibinfo{year}{2013}\natexlab{}.
\newblock \showarticletitle{Lasso {Screening} {Rules} via {Dual} {Polytope} {Projection}}. In \bibinfo{booktitle}{\emph{Advances in {Neural} {Information} {Processing} {Systems}}}, Vol.~\bibinfo{volume}{26}. \bibinfo{publisher}{Curran Associates, Inc.}
\newblock
\urldef\tempurl%
\url{https://papers.nips.cc/paper/2013/hash/8b16ebc056e613024c057be590b542eb-Abstract.html}
\showURL{%
\tempurl}


\bibitem[West et~al\mbox{.}(2001)]%
        {West2001}
\bibfield{author}{\bibinfo{person}{Mike West}, \bibinfo{person}{Christian Blanchette}, \bibinfo{person}{Holly Dressman}, \bibinfo{person}{Erich Huang}, \bibinfo{person}{Sue Ishida}, \bibinfo{person}{Rainer Spang}, \bibinfo{person}{Hannah Zuzan}, \bibinfo{person}{Joseph~A. Olson}, \bibinfo{person}{Jeffrey~R. Marks}, {and} \bibinfo{person}{Joseph~R. Nevins}.} \bibinfo{year}{2001}\natexlab{}.
\newblock \showarticletitle{Predicting the clinical status of human breast cancer by using gene expression profiles}.
\newblock \bibinfo{journal}{\emph{Proceedings of the National Academy of Sciences}} \bibinfo{volume}{98}, \bibinfo{number}{20} (\bibinfo{year}{2001}), \bibinfo{pages}{11462--11467}.
\newblock


\bibitem[Xiang et~al\mbox{.}(2016)]%
        {xiang2016screening}
\bibfield{author}{\bibinfo{person}{Zhen~James Xiang}, \bibinfo{person}{Yun Wang}, {and} \bibinfo{person}{Peter~J Ramadge}.} \bibinfo{year}{2016}\natexlab{}.
\newblock \showarticletitle{Screening tests for lasso problems}.
\newblock \bibinfo{journal}{\emph{IEEE transactions on pattern analysis and machine intelligence}} \bibinfo{volume}{39}, \bibinfo{number}{5} (\bibinfo{year}{2016}), \bibinfo{pages}{1008--1027}.
\newblock


\bibitem[Xu et~al\mbox{.}(2011)]%
        {xu2011sparse}
\bibfield{author}{\bibinfo{person}{Huan Xu}, \bibinfo{person}{Constantine Caramanis}, {and} \bibinfo{person}{Shie Mannor}.} \bibinfo{year}{2011}\natexlab{}.
\newblock \showarticletitle{Sparse algorithms are not stable: A no-free-lunch theorem}.
\newblock \bibinfo{journal}{\emph{IEEE transactions on pattern analysis and machine intelligence}} \bibinfo{volume}{34}, \bibinfo{number}{1} (\bibinfo{year}{2011}), \bibinfo{pages}{187--193}.
\newblock


\bibitem[Yeoh et~al\mbox{.}(2002)]%
        {Yeoh2002}
\bibfield{author}{\bibinfo{person}{Eng-Juh Yeoh}, \bibinfo{person}{Megan~E. Ross}, \bibinfo{person}{Sheila~A. Shurtleff}, \bibinfo{person}{William~K. Williams}, \bibinfo{person}{Divya Patel}, \bibinfo{person}{Ramy Mahfouz}, \bibinfo{person}{Frederick~G. Behm}, \bibinfo{person}{Susana~C. Raimondi}, \bibinfo{person}{Mary~V. Relling}, \bibinfo{person}{Asha Patel}, \bibinfo{person}{Cheng Cheng}, \bibinfo{person}{Dario Campana}, \bibinfo{person}{David Wilkins}, \bibinfo{person}{Xiaowei Zhou}, \bibinfo{person}{Jia Li}, \bibinfo{person}{Han Liu}, \bibinfo{person}{Ching-Hon Pui}, \bibinfo{person}{William~E. Evans}, \bibinfo{person}{Clifford Naeve}, \bibinfo{person}{Lawrence Wong}, {and} \bibinfo{person}{James~R. Downing}.} \bibinfo{year}{2002}\natexlab{}.
\newblock \showarticletitle{Classification, subtype discovery, and prediction of outcome in pediatric acute lymphoblastic leukemia by gene expression profiling}.
\newblock \bibinfo{journal}{\emph{Cancer Cell}} \bibinfo{volume}{1}, \bibinfo{number}{2} (\bibinfo{year}{2002}), \bibinfo{pages}{133--143}.
\newblock


\bibitem[Yoo et~al\mbox{.}(2003)]%
        {yoo2003slurm}
\bibfield{author}{\bibinfo{person}{Andy~B Yoo}, \bibinfo{person}{Morris~A Jette}, {and} \bibinfo{person}{Mark Grondona}.} \bibinfo{year}{2003}\natexlab{}.
\newblock \showarticletitle{Slurm: Simple linux utility for resource management}. In \bibinfo{booktitle}{\emph{Workshop on job scheduling strategies for parallel processing}}. Springer, \bibinfo{pages}{44--60}.
\newblock


\bibitem[Yuan et~al\mbox{.}(2012)]%
        {Yuan2012}
\bibfield{author}{\bibinfo{person}{Guo-xun Yuan}, \bibinfo{person}{Chia-Hua Ho}, {and} \bibinfo{person}{Chih-jen Lin}.} \bibinfo{year}{2012}\natexlab{}.
\newblock \showarticletitle{An improved {GLMNET} for {L1}-regularized logistic regression}.
\newblock \bibinfo{journal}{\emph{Journal of Machine Learning Research}}  \bibinfo{volume}{13} (\bibinfo{year}{2012}), \bibinfo{pages}{1999--2030}.
\newblock
\urldef\tempurl%
\url{https://doi.org/10.1145/2020408.2020421}
\showDOI{\tempurl}
\newblock
\shownote{Publisher: ACM Press Place: New York, New York, USA ISBN: 9781450308137}.


\end{thebibliography}
